\documentclass{colt2026_arxiv}

\title{Optimal Reconstruction from Linear Queries}
\usepackage{times}
\usepackage{mdframed}
\usepackage{xcolor}
\usepackage{cleveref}
\usepackage{enumitem}

\coltauthor{%
 \Name{Yuval Filmus} \Email{yuvalfi@cs.technion.ac.il}\\
 \addr Technion -- Israel Institute of Technology
 \AND
 \Name{Shay Moran} \Email{smoran@technion.ac.il} \\
 \addr Technion -- Israel Institute of Technology\\
 \addr Google Research
 \AND
 \Name{Elizaveta Nesterova} \Email{elizavetan@campus.technion.ac.il}\\
 \addr Technion -- Israel Institute of Technology
}

\newtheorem{maintheorem}[theorem]{Main Theorem}

\newenvironment{restatedtheorem}[1]
  {\par\noindent\textbf{Main Theorem~{#1}~restatement.}\itshape}
  {\par}

\newenvironment{rtheorem}[1]
  {\par\noindent\textbf{Theorem~{#1}~restatement.}\itshape}
  {\par}

\newtheorem*{theorem*}{Theorem}
\newtheorem*{lemma*}{Lemma}
\newtheorem*{notation*}{Notation}
\newtheorem*{proposition*}{Proposition}
\newtheorem*{corollary*}{Corollary}
\newtheorem*{claim*}{Claim}
\newtheorem*{definition*}{Definition}
\newtheorem*{assumption*}{Assumption}
\newtheorem*{question*}{Question}
\newtheorem*{remark*}{Remark}
\newtheorem*{example*}{Example}

\newcommand{\R}{\mathbb{R}}
\newcommand{\NN}{\mathbb{N}}
\newcommand{\ZZ}{\mathbb{Z}}
\newcommand{\cX}{\mathcal{X}}
\newcommand{\cY}{\mathcal{Y}}
\newcommand{\cR}{\mathcal{R}}

\newcommand{\cA}{\mathcal{A}}
\newcommand{\cN}{\mathcal{N}}
\newcommand{\cV}{\mathcal{V}}

\newcommand{\dist}{\mathrm{dist}}

\newcommand{\OPT}{\mathrm{OPT}}

\newcommand{\improper}{\mathrm{improper}}

\newcommand{\Witness}{\mathrm{Witness}}
\newcommand{\rad}{\mathtt{rad}}

\newcommand{\tnew}{{t_\mathrm{new}}}

\newcommand{\Excess}{\mathsf{ExcessErr}}
\newcommand{\ADV}{\mathsf{ADV}}
\newcommand{\RC}{\mathsf{RC}}

\newcommand{\Jung}{\mathrm{Jung}}
\newcommand{\diam}{\mathrm{diam}}
\newcommand{\vvv}{\overrightarrow}

\begin{document}

\maketitle

\begin{abstract}
We study the problem of reconstructing an unknown point in $\mathbb{R}^d$ from approximate linear queries. This setting arises naturally in applications ranging from low-dimensional remote sensing and signal recovery to high-dimensional data analysis and privacy-sensitive inference. Our main goal is to characterize the optimal \emph{reconstruction error} as a function of the number of queries $T$, the ambient dimension $d$, and the noise parameter $\delta$. 

We first analyze the limit $T \to \infty$ and show that the optimal reconstruction error converges to the explicit value $\sqrt{2d/(d+1)}\,\delta$, which plays a role analogous to the Bayes optimal error in supervised learning. 
When the dimension is fixed, we show that the excess error above this limit decays \emph{doubly exponentially} fast as $T \to \infty$, a rate that is significantly faster than those typically encountered in learning curves. 
When the dimension grows, we show that a number of queries on the order of $\exp(d)$ is necessary and sufficient to achieve vanishing excess error. 
Finally, we introduce and analyze an improper variant of the reconstruction problem.

From a technical perspective, our main contribution is a generalization of Jung’s theorem (1901). The classical theorem bounds the maximum possible radius of a set of diameter 1 and characterizes extremal bodies. Our generalization provides a robust variant that characterizes near-extremal bodies and is proved via geometric and dynamical arguments exploiting symmetry and Lie group actions.
\end{abstract}

\section{Introduction}

In this work we study the following question:
\begin{quote}
\emph{How accurately can an unknown point \(x^* \in \R^d\) be reconstructed using linear measurements?}
\end{quote}
We formalize this question via a  \emph{linear reconstruction game}. In this game, an adversary holds a secret point $x^* \in \R^d$, and at each round:
\begin{enumerate}[topsep=3pt,itemsep=1pt,parsep=0pt]
    \item The reconstructor submits a linear query in the form of a unit vector 
    \(v \in \R^d\), \(\|v\|_2 = 1\).\footnote{Equivalently, we can allow arbitrary vectors $v$, and require the error to be at most $\delta \|v\|_2$. Without this normalization, the reconstructor can cheat by scaling $v$ by a large amount, effectively reducing the error.} 
    \item The adversary replies with an adversarially chosen value $r$ satisfying
    \[
     |r - \langle v, x^* \rangle| \le \delta,
    \]
    where $\delta>0$ is a fixed additive noise parameter, and $\langle v, x \rangle$ is the standard inner product.
\end{enumerate}

\noindent The reconstructor can choose its queries adaptively, based on the answers to previous queries. After a given number of rounds $T$, it outputs an estimate \(\hat x\), aiming to minimize the \(\ell_2\)-distance \(\|\hat x -  x^*\|_2\).

The linear reconstruction game captures phenomena that arise naturally in both low- and high-dimensional settings.
In low-dimensional geometric problems, it models adaptive sensing and signal acquisition, where measurement directions are chosen sequentially in order to localize an unknown signal or object in space.
Related formulations also arise in data analysis and privacy --- often, but not exclusively, in high dimensions --- where linear queries are used to probe a dataset, and noise is added to limit reconstruction of sensitive information.

A discrete variant of this problem, in which both the hidden point and the queries are restricted to the Boolean cube $\{0,1\}^n$, was introduced by \citet{DinurN03}, and played a foundational role in the development of differential privacy.
Subsequent work in private and adaptive data analysis, as well as in the statistical query model, has studied this setting extensively, primarily from the perspective of preventing reconstruction
\citep{DBLP:conf/stoc/BassilyNSSSU16, DBLP:conf/nips/DworkFHPRR15, DBLP:conf/stoc/DworkFHPRR15, Reyzin2020StatisticalQA}.
Related questions also appear in the adaptive sensing literature \citep{AriasCastroCD13, HauptCastroNowak09}, though most of that work assumes stochastic noise, whereas we consider adversarial noise.

We study the optimal minimax reconstruction error against an \emph{arbitrary} adversary:
\begin{mdframed}
\begin{center}
\textbf{Main Question}
\end{center}
\noindent
What is the optimal reconstruction error that can be guaranteed against an arbitrary (worst-case) adversary, as a function of the error parameter $\delta$, the number of rounds $T$, and the dimension~$d$?
\end{mdframed}

From a learning-theoretic perspective, this question can be viewed as a natural instance of interactive learning, in which a learner adaptively queries an oracle and seeks to reconstruct an unknown target from noisy feedback.
Here, the secret point $x^*$ plays the role of the concept to be learned.

\section{Main results}
In this section, we outline the key definitions and results. More detailed definitions and statements are deferred to Section~\ref{sec:formal-definitions}.

The \emph{optimal reconstruction error} is the smallest worst-case loss achievable by a reconstructor, i.e., the smallest worst-case loss that the “best” reconstructor can achieve.  We denote the optimal reconstruction error for noise level $\delta > 0$, dimension $d$, and $T$ interaction rounds by $\OPT_d(T, \delta)$.
Our first main theorem identifies the reconstruction error achievable (in the limit) given an unlimited number of rounds.

\begin{mdframed}[
  linewidth=0.6pt,
  roundcorner=0pt,
  backgroundcolor=gray!0,
  nobreak=true
]
\begin{maintheorem}[Asymptotic optimal error]
\label{thm:asymptotic}

\[
 \OPT_d(\infty, \delta) := \lim_{T \to \infty} \OPT_d(T, \delta) = \sqrt{\frac{2d}{d+1}} \delta.
\]
\end{maintheorem}
\end{mdframed}

From a learning-theoretic viewpoint, the limiting value \(\OPT_d(\infty,\delta)\) is analogous to the \emph{Bayes-optimal} error: it is the best performance achievable under the interaction and noise constraints of the model.
This motivates studying the \emph{excess reconstruction error},
\[
\Excess_d(T,\delta)
\;:=\;
\OPT_d(T,\delta) - \OPT_d(\infty,\delta).
\]
Our second main theorem, which is also our main technical contribution, shows
that \(\Excess_d(T,\delta)\) decays at a doubly exponential rate in \(T\) when \(d \ge 2\) (when \(d = 1\), there is only one unit vector $v = 1$ up to sign, and so one round suffices).

\begin{mdframed}[
  linewidth=0.6pt,
  roundcorner=0pt,
  backgroundcolor=gray!0,
  nobreak=true
]
\begin{maintheorem}[Doubly-exponential decay of excess error]
\label{thm:excess-bounds}

For every dimension $d\ge 2$ there exists \(T'_d\) such that for every $T\ge T'_d$:
\[
  \Excess_d(T, \delta) = 
  2^{-2^{\Theta_d(T)}} \delta.
\]
\end{maintheorem}
\end{mdframed}

The linear dependence on \(\delta\) in Main Theorems~\ref{thm:asymptotic} and \ref{thm:excess-bounds} is not accidental: in fact, the reconstruction game is invariant under rescaling \(\delta\). We elaborate on this further in the Technical Overview and prove it in Section~\ref{sec:formal-definitions}.

The proof of Main Theorem~\ref{thm:excess-bounds} relies on a reconstruction algorithm with a natural two-phase structure.
In the first phase, which can be viewed as a preprocessing step, the reconstructor queries a sufficiently dense (but oblivious) set of directions in order to obtain a coarse geometric localization of the hidden point.
In the second phase, the reconstructor exploits this localization to adaptively refine its estimate, leading to the doubly-exponential decay of the excess error.
This rapid refinement crucially relies on our robust Jung theorem: a robust variant of Jung’s theorem, which identifies a small set of particularly informative directions whose queries sharply reduce the remaining uncertainty.
We elaborate on this robust variant of Jung’s theorem and explain how it is leveraged algorithmically in the proof overview in the next section.

The cost of the first, preprocessing phase grows rapidly with the dimension $d$, as it requires querying a dense net of directions on the unit sphere.

This naturally raises the question of whether such a dependence on the dimension is inherent.
Our final main theorem answers this question in the affirmative and identifies $T=\exp(d)$ as a threshold for achieving vanishing excess error, in the following sense:

\begin{mdframed}[
  linewidth=0.6pt,
  roundcorner=0pt,
  backgroundcolor=gray!0,
  nobreak=true
]
\begin{maintheorem}[Dimension-dependent query budgets]
\label{thm:dim-bounds}
Let \(T : \NN \to \NN\) be a query budget as a function of the dimension \(d\).

\begin{enumerate}
  \item If \(T(d) = 2^{o(d)}\) is subexponential in \(d\) then
  \[
    \Excess_d(T(d),\delta) \xrightarrow{d \to \infty} \infty.
  \]
  \item If \(T(d) = 2^{\omega(d)}\) is superexponential in \(d\) then
  \[
    \Excess_d(T(d),\delta) \xrightarrow{d \to \infty} 0.
  \]
\end{enumerate}
\end{maintheorem}
\end{mdframed}

\paragraph{Improper reconstructors.}
An accurate reconstruction of the point $x^*$ immediately yields accurate answers to future linear queries: indeed, if $\|\hat{x}-x^*\|_2 \le \varepsilon$, then for every unit vector $v$,
\[
\bigl|\langle v,\hat{x}\rangle - \langle v,x^*\rangle\bigr|
= \bigl|\langle v,\hat{x}-x^*\rangle\bigr|
\le \|\hat{x}-x^*\|_2
\le \varepsilon .
\]
This observation motivates the study of \emph{improper reconstructors}, whose goal is to output a function
$\hat{G}\colon S^{d-1} \to \mathbb{R}$ (where $S^{d-1}$, the unit sphere in $\R^d$, consists of all unit vectors)
minimizing
\[
\sup_{v \in S^{d-1}} \bigl|\hat{G}(v) - \langle v,x^*\rangle\bigr|.
\]
We analyze this setting as well; here we briefly summarize the results and highlight the key differences from the proper setting, deferring complete statements and proofs to Section~\ref{sec:improper}.

In the improper setting, the optimal reconstruction error converges, as $T \to \infty$, to the value~$\delta$, which is strictly smaller than the optimal limit in the proper setting, given by $\sqrt{\frac{2d}{d+1}}\,\delta \approx \sqrt{2}\,\delta$.
While the excess error in the proper setting decays doubly exponentially fast in $T$, convergence in the improper setting is only polynomial.
Specifically, for any fixed dimension $d \ge 2$,
\[
\OPT^{\improper}_d(T,\delta)
=
\bigl(1 + \Theta_d(T^{-2/(d-1)})\bigr)\,\delta .
\]
At first glance, it may seem paradoxical that improper reconstruction converges more slowly, since every proper reconstructor induces an improper one via $\hat{G}(v)=\langle v,\hat{x}\rangle$.
The resolution is that improper reconstruction converges to a strictly smaller limiting error, and the slower rate reflects the greater difficulty of approaching this stronger benchmark.

This improvement, however, comes at a cost.
While a proper reconstructor outputs a single point in $\mathbb{R}^d$, an improper reconstructor outputs a function on $S^{d-1}$, which is an infinite object and raises substantial space complexity concerns.
In particular, the improper reconstructor underlying our bounds must retain the entire interaction history in order to answer future queries.

Finally, as in the proper setting, the dependence of the hidden constants on the dimension is exponential; in particular, Main Theorem~\ref{thm:dim-bounds} applies to the improper setting as well.

\paragraph{Comparison with prior work.}
Our work is closely connected to the reconstruction problem studied
in~\cite{moranreconstruction}. Both works consider interactive reconstruction
from approximate answers, but the query models are different. In
\cite{moranreconstruction}, the unknown object is a point \(x^\star\) in a
metric space \((X,d)\), and a reconstruction strategy adaptively queries
points \(q\in X\) and receives approximate answers to the distance queries
\(d(q,x^\star)\). The goal is to output a point as close as possible to
\(x^\star\), in the worst case, given the ambiguity left by the answers. In the
present paper, the unknown object is a vector, the queries are linear
functionals, and the answers are approximate values of these functionals.

At the level of limiting guarantees, our Main
Theorem~\ref{thm:asymptotic} is a direct analogue of Theorem~2
in~\cite{moranreconstruction}: both identify the optimal reconstruction error
in the limit of infinitely many queries. In the Euclidean setting, the two
results yield identical guarantees and rely on the same underlying geometric
invariants, as discussed further in the Technical Overview.

The main difference is quantitative. The work~\cite{moranreconstruction}
shows that the limiting error need not be attainable in finitely many rounds
and leaves open the rate of convergence to the limit. Our main contribution is
to determine this rate for linear queries: we prove matching upper and lower
bounds in fixed dimension (Main Theorem~\ref{thm:excess-bounds}) and
characterize the dependence on the dimension (Main
Theorem~\ref{thm:dim-bounds}). It remains an interesting question whether the
upper-bound techniques developed here can be adapted to the distance-query
framework of~\cite{moranreconstruction}.

\paragraph{Organization.}
Section~\ref{sec:proof-overview} provides a proof overview and highlights the key technical contributions of our work. Section~\ref{sec:formal-definitions} gives the formal model and basic properties of the reconstruction game. The proofs of the main theorems are given in Section~\ref{sec:proofs}, and the results for improper reconstruction are given in Section~\ref{sec:improper}. We conclude with open directions in Section~\ref{sec:open-directions}. The technical tools underlying the proofs are developed in Appendices~\ref{app:Jung_approx}, \ref{app:step_2}, \ref{app:covering-number} and ~\ref{app:rotation_lemmas}.

\section{Technical Overview}
\label{sec:proof-overview}

\noindent
We organize the proof overview as follows. We first collect several technical tools that are useful for explaining the key ideas underlying our proofs. We then outline the proofs of the main theorems, each in a separate subsection.

\paragraph{Feasible region.} 
Assume that on query \(v\), the adversary returns an answer \(r \in \R\).
What information about the secret point is now available to the reconstructor?

\smallskip
\noindent
The reconstructor can conclude only that the secret \(x^*\) lies in \( \{ x \in \R^d : |\langle x, v\rangle - r| \le \delta \}\), which is a strip
orthogonal to \(v\) of width \(2\delta\). Therefore, after observing the whole interaction transcript (i.e., the sequence of queries \(v_i\) and answers~\(r_i\)), the reconstructor only learns that the secret lies in the \emph{intersection} of the corresponding strips. Formally, the secret must belong to the set of all points~\(x~\in~\R^d\) that are
consistent with the transcript up to noise level \(\delta>0\):
\begin{equation}
\label{eq:feasible}
\Phi_T
:= \bigl\{ x \in \R^d \ \big|\  |\langle v_i, x\rangle - r_i|
\le \delta \ \text{ for all } i\in[T] \bigr\}.
\end{equation}
We call \(\Phi_T\) the \emph{feasible region}.

What worst-case error can the reconstructor guarantee after observing the transcript $\{(v_i,r_i)\}_{i\in[T]}$, that is, after identifying the feasible set $\Phi_T$?
Suppose the reconstructor outputs an estimate $\hat x_T$.
In the worst case, the true point $x^*$ may be a point in $\Phi_T$ that maximizes its distance from $\hat x_T$. Consequently, to minimize the worst-case error, the reconstructor should choose an output that minimizes the maximum distance to all points in $\Phi_T$.
This motivates the following definition.

For a set $A \subseteq \mathbb{R}^d$, define its \emph{radius} (also known as the \emph{Chebyshev radius}) by
\[
\rad(A)
\;=\;
\inf_{x \in \mathbb{R}^d}
\;
\sup_{y \in A}
\;
\|x - y\|_2 .
\]
Equivalently, $\rad(A)$ is the radius of the minimal enclosing ball of $A$. See \Cref{fig:proper-feasible} for a visualization.

\begin{remark*}
The feasible region is a central object in this game. It plays a key role throughout all of our theorems: both the reconstructor's algorithm and the adversary's strategy exploit its geometric properties. In particular, we formulate the guarantees and the convergence rates by tracking how the radius of the feasible region evolves with the transcript.
\end{remark*}

\begin{figure}[t]
\begin{minipage}[t]{0.47\textwidth}
\centering
    \includegraphics[width=0.9\textwidth]{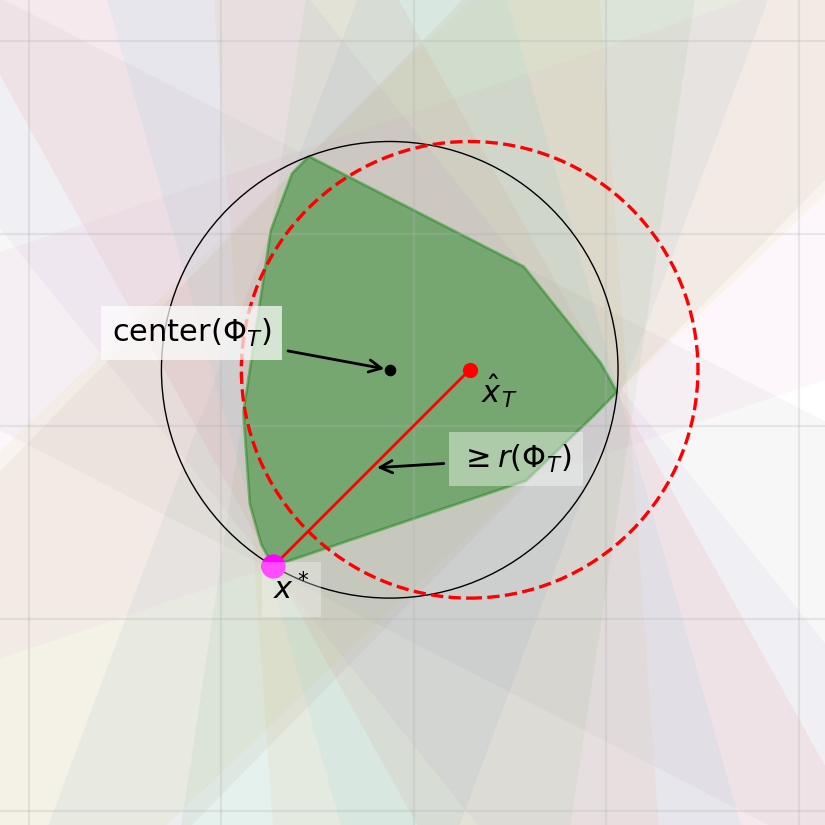}
    \caption{A secret point \(x^*\) and an output \(\hat x_T\) such that \(\|\hat x_T - x^*\| \ge \rad(\Phi_T)\).}
    \label{fig:proper-feasible}
\end{minipage}\hfill
\begin{minipage}[t]{0.47\textwidth}
\centering
    \includegraphics[width=0.9\textwidth]{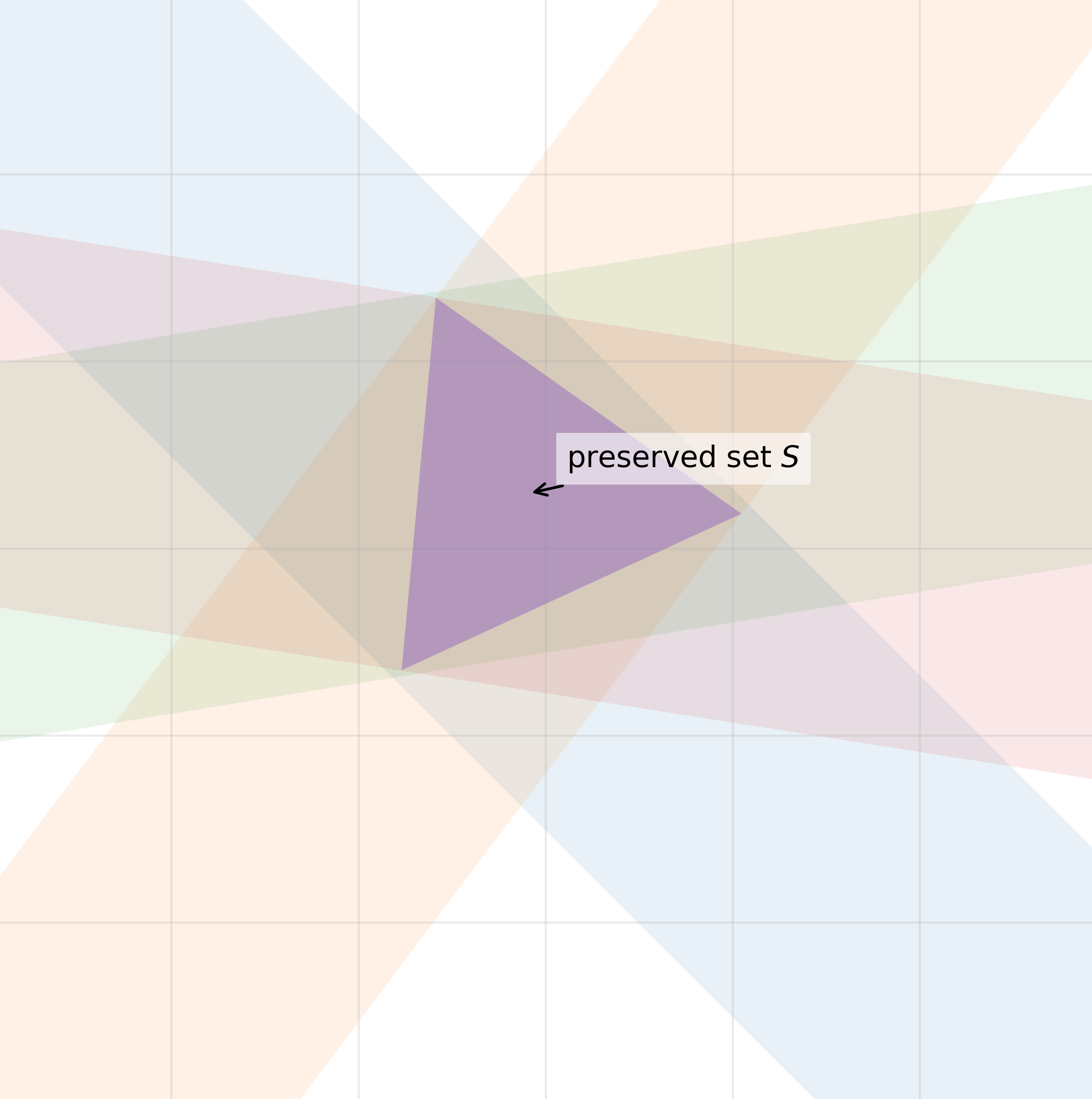}
    \caption{Adversary strategy to maintain \(S\) of \(\diam \ S \le 2\delta\) inside \(\Phi_T\).}
    \label{fig:adv_strategy}
\end{minipage}
\end{figure}

\paragraph{Scale invariance.}
As a final preparatory fact before turning to the proof overview, we record a simple scaling property of the optimal reconstruction error that will be useful throughout the paper. Recall that \(\OPT_d(T,\delta)\) denotes the worst-case reconstruction error achievable by the reconstructor under noise level \(\delta>0\).
For every $T \in \mathbb{N}$ and every $\delta > 0$, the optimal error satisfies
\begin{equation}
\label{eq:noise_scale}
    \OPT_d(T,\delta) = \delta \, \OPT_d(T,1).
\end{equation}
Intuitively, this is the case since we can obtain an algorithm for error rate $\delta$ from an algorithm for error rate $1$ by scaling everything by $\delta$.
The formal proof appears in Lemma~\ref{lem:scaling-proof}.

This scale invariance allows us to simplify several arguments.
In particular, in the proof of Main Theorem~\ref{thm:excess-bounds} we work with normalized noise $\delta = 1/2$ without loss of generality.
Similarly, in Main Theorem~\ref{thm:dim-bounds}, we set $\delta = \Theta(\sqrt{\ln T})$, which simplifies the analysis by allowing us to work with standard Gaussian vectors.

The rest of the technical overview outlines the key ideas behind the three main theorems, in order.

\subsection{Asymptotic error: Main Theorem~\ref{thm:asymptotic}}

This subsection outlines the key ideas in the proof of Main Theorem~\ref{thm:asymptotic}. The approach used in the proof of Main Theorem~\ref{thm:asymptotic} is not new; a closely related approach was used in~\cite{moranreconstruction} for the distance reconstruction game on arbitrary metric spaces. In particular, the key step in their proof of Theorem~2 is to bound the Chebyshev radius of the feasible region via a bound on its diameter.  We nevertheless present these ideas here, since they form the basis of the intuition behind Main Theorems~\ref{thm:excess-bounds} and~\ref{thm:dim-bounds}. Recall that Main Theorem~\ref{thm:asymptotic} identifies the asymptotic optimal reconstruction error:
\[
 \OPT_d(\infty, \delta) := \lim_{T \to \infty} \OPT_d(T, \delta) = \sqrt{\frac{2d}{d+1}} \delta.
\]
\noindent

\smallskip

\noindent \textit{Upper bound.}
Assume that the reconstructor queries \emph{all} directions $v \in S^{d-1}$; we can simulate this by querying a dense enough net. We first bound the \emph{diameter} of the resulting feasible region $\Phi$:
\[
\diam(\Phi) \le 2\delta .
\]
Indeed, for any two points \(A,B \in \Phi\), we can consider the direction \(v = \frac{A-B}{\|A - B\|_2}\). By definition of the feasible region, \(|\langle A - B, v \rangle| \le 2\delta\), and so \(\|A - B\|_2 \le 2\delta\).

We derive a bound on the Chebyshev radius of the feasible region using Jung's theorem; for a modern proof, see \cite[Theorem~3.3]{gruber2007convex}. 

\begin{theorem}[Jung, 1901]
\label{thm:Jung}
For every bounded set \(S \subseteq \R^d\) with diameter \(D\), its Chebyshev radius satisfies \(\rad(S) \le \Jung_d\,D\), where \(\Jung_d = \sqrt{\tfrac{d}{2(d+1)}}\) is known as \emph{Jung's constant}.
\end{theorem}
\noindent
In particular, since the feasible region has diameter at most \(2\delta\), Jung's theorem yields the upper bound 
\[\rad(\Phi)\le \Jung_d\, \cdot 2\delta = \sqrt{\frac{2d}{d+1}} \delta.
\] 

\smallskip

\noindent \textit{Lower bound.}
We design an adversary which ensures that the feasible region contains some fixed set \(S\) whose Chebyshev radius is $\Jung_d \cdot 2\delta$. Motivated by the upper bound, we pick \(S\) to be the vertex set of a regular simplex of edge length \(2\delta\), which is tight for Jung's theorem.

The adversary proceeds as follows. Given an arbitrary direction \(v\), it considers the interval \(\{\langle s,v \rangle : s \in S\}\). Since \(S\) has diameter \(2\delta\), this interval has length at most \(2\delta\), and the adversary reports its midpoint. This way, it ensures that the entire set \(S\) lies in the strip corresponding to \(v\). See \Cref{fig:adv_strategy} for a visualization.

\subsection{Excess error: Main Theorem~\ref{thm:excess-bounds}}

The aim of this section is to outline the proof of Main Theorem~\ref{thm:excess-bounds}. Recall that this result determines the rate at which the excess error decays:
\begin{restatedtheorem}{\ref{thm:excess-bounds}}
For every dimension \(d\ge 2\) there exists \(T'_d\) s.t.\ for every \(T\ge T'_d\),
\[
  \Excess_d(T, \delta) =
  2^{-2^{\Theta_d(T)}} \delta.
\]
\end{restatedtheorem}
\noindent
To prove this result, we first note that by the scaling property we may assume throughout that the noise is normalized to \(\delta=\tfrac12\). Under this normalization, it suffices to show that for every dimension~\(d \ge 2\) there exists \(T'_d\) such that for every \(T \ge T'_d\),
\[
  \OPT_d(T, 1/2)
  = \Jung_d + 2^{-2^{\Theta_d(T)}},
\]
where \(\Jung_d = \OPT_d(\infty, 1/2) = \sqrt{\frac{d}{2(d+1)}}\) is Jung’s constant (see Theorem~\ref{thm:Jung}).

\smallskip

The doubly exponential rate in Main Theorem~\ref{thm:excess-bounds} is unusual. \cite{moranreconstruction} proved a doubly exponential lower bound in a related setting (distance queries), which we repurpose for our setting, but they didn't prove a matching upper bound. The argument underlying the proof of the upper bound in Main Theorem~\ref{thm:asymptotic} - querying a dense enough net of directions - only yields an upper bound of the form \(O(T^{-2/(d-1)} \delta)\). Instead, the proof of the upper bound in Main Theorem~\ref{thm:excess-bounds}, which is our main technical contribution, uses a boosting step  which reduces an excess error of~$\epsilon$ to an excess error of $O(\epsilon^2)$ using constantly many queries (the exact number depending on the dimension).

The main idea of the proof of Main Theorem~\ref{thm:asymptotic} is to reduce the diameter of the feasible region and then apply Jung's theorem to deduce a bound on the Chebyshev radius. This is lossy since there is no guarantee that Jung's theorem is tight for the feasible region. Instead, in the proof of Main Theorem~\ref{thm:excess-bounds} we focus on eliminating subsets of the feasible region with large Chebyshev radius.

Suppose that the feasible region $\Phi_T$ at time $T$ has Chebyshev radius at least \(\Jung_d + \epsilon\). By Carath\'eodory's theorem, this is witnessed by some set $S \subseteq \Phi_T$ of size at most $d + 1$ whose Chebyshev radius is at least \(\Jung_d + \epsilon\). We can think of $S$ as forming the vertex set of a simplex. If we query the directions of all edges of the simplex, then $S$ can no longer be inside the feasible region, potentially decreasing the Chebyshev radius of the feasible region.

The issue with this strategy is that there could be many such sets, and we cannot afford to query the directions corresponding to all of them. To overcome this difficulty, we prove a robust version of Jung's theorem, which shows that all such sets $S$ are close to a single vertex set $\Delta$. Querying the directions corresponding to the edges of $\Delta$ will reduce the Chebyshev radius of the feasible region to $1 + O(\epsilon^2)$.

\begin{figure}[t]
\begin{center}
\begin{minipage}[t]{0.9\textwidth}
\centering\includegraphics[width=0.25\textwidth]{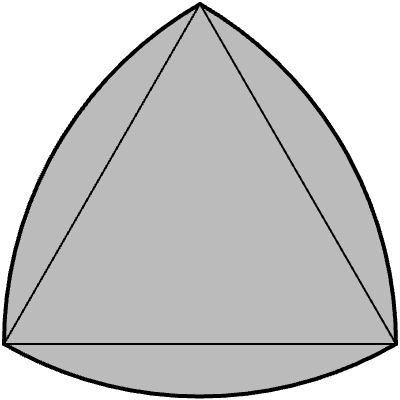}
\caption{Reuleaux triangle of diameter $1$ and Chebyshev radius $\Jung_2 = \sqrt{\frac13}$. \\The vertices of the highlighted triangle constitute the unique witness.}
\label{fig:reuleaux}
\end{minipage}\hfill
\end{center}
\end{figure}

\subsubsection{A robust Jung's theorem}

Jung's theorem states that a set $S \subseteq \R^d$ with diameter $1$ has Chebyshev radius at most $\Jung_d$. Furthermore, if the Chebyshev radius is exactly $\Jung_d$ then $S$ contains the vertex set of a regular simplex $\Delta$ with edge length~$1$, which is moreover unique. However, $S$ itself could contain points beyond $\Delta$. For example, $S$ could be a Reuleaux triangle (see Figure~\ref{fig:reuleaux}).

The robust version of Jung's theorem describes the structure of sets $S \subseteq \R^d$ with diameter $1+\beta$ whose Chebyshev radius is at least $\Jung_d$. While we cannot say that such sets are close to a regular simplex (this is false even when $\beta = 0$), we are able to characterize the ``extremal parts'' of~$S$, which we call witnesses. A subset $W \subseteq S$ is a \emph{witness} for $S$ if $|W| = d + 1$ and
\[
 \rad(W) \ge \Jung_d.
\]
We would like to say that the witnesses cluster around one another. We formalize this using the following definition: two sets $W_1,W_2$ are \emph{$r$-close} if
\[
 W_1 \subset W_2 + B(r) \qquad \text{and} \qquad W_2 \subset W_1 + B(r),
\]
where $B(r)$ is the ball of radius $r$ and addition is Minkowski sum; equivalently, $W + B(r)$ consists of all points which are at distance at most $r$ from some point in $W$.

We can now state the robust version of Jung's theorem, which describes the structure of sets which are almost tight for Jung's theorem.

\begin{theorem}[Robust Jung Theorem]
\label{thm:robust-jung}
Let $d \ge 1$ and let $S \subset \mathbb{R}^d$ be a closed set satisfying
\[
\diam(S) \le 1+\beta
\qquad\text{and}\qquad
\rad(S) \ge \Jung_d ,
\]
for some $0 \le \beta \le \beta_d$, where $\beta_d>0$ depends only on $d$.
Then, there exists a set of points
\[
\Delta = \{x_0,x_1,\dots,x_d\} \subset \mathbb{R}^d
\]
forming the vertex set of a regular simplex of edge length $1$ such that every witness \(W \subseteq S\) is \(C_d\beta\)-close to~\(\Delta\),
for a constant $C_d>0$ depending only on $d$.
\end{theorem}

We first explain how to use this theorem to prove the upper bound in Main Theorem~\ref{thm:excess-bounds}, and then how it suggests the adversary strategy for the lower bound. Finally, we sketch the proof of the theorem itself. 

\subsubsection{Upper bounds on excess error}
\label{sec:upper_bound}

To obtain a doubly exponential upper bound, it suffices to construct a reconstructor strategy and constants \(T'_d\) and \(c_d>0\) such that for all \(T\ge T'_d\),
\[
\rad(\Phi_T)\le \Jung_d + 2^{-2^{c_d\cdot (T-T'_d)}}.
\]
In particular, for $T \ge 2T'_d$ we have $T-T'_d \ge T/2$, and therefore \(\rad(\Phi_T)\le \Jung_d + 2^{-2^{(c_d/2)\,T}}\), which is a doubly exponential convergence rate in $T$ (up to constants in the exponent). The formal proof of this Theorem is given in Section~\ref{sec:upper-bounds}

\smallskip

The algorithm proceeds in two phases. In the preprocessing phase of \(T'_d\) queries, the reconstructor bounds the diameter of the feasible region so that Robust Jung Theorem (Theorem~\ref{thm:robust-jung}) applies to it. After this point, the algorithm enters a refinement phase. Each refinement step consists of a fixed batch of \(\binom{d+1}{2}\) queries, and has the effect of squaring the current excess error (up to dimension-dependent constants). Iterating this refinement yields the desired doubly exponential convergence.

The key geometric object helping to control the Chebyshev diameter of the feasible region is the \emph{witness core}:
\begin{equation}
\Witness_t \;:=\; \bigcup_{\substack{W \subset \Phi_t:\\
W \text{ is a witness } }} W, 
\end{equation}
that is, the union of all witnesses. Intuitively, $\Witness_t$ is the part of the feasible region that prevents the Chebyshev radius from dropping below $\Jung_d=\OPT_d(\infty,1/2)$. Indeed, $\rad(\Phi_t)<\Jung_d$ holds if and only if
$\Witness_t$ is empty.

By Lemma~\ref{lem:minimum-witness-cardinality} (via Helly's theorem),
there exists a finite set \(W \subseteq \Phi_t\) of size \(d+1\) such that \(\rad(\Phi_t) = \rad(W)\). Assuming \(\rad(\Phi_t) \ge \Jung_d\), the set \(W\) is a witness for
\(\Phi_t\), and hence \(W \subseteq \Witness_t\) by definition of the witness
core. Applying Jung’s inequality to \(W\), we obtain
\[
\rad(\Phi_t)
= \rad(W)
\le \Jung_d \,\diam(W)
\le \Jung_d \,\diam(\Witness_t).
\]
Consequently, it suffices to show that \(\diam(\Witness_t)\) converges to \(1\)
doubly exponentially fast as~\(t\) grows. Accordingly, the reconstructor’s goal is to reduce the diameter of the witness core.

\paragraph{Preprocessing.}
At the beginning of the game, the reconstructor obliviously queries a sufficiently dense covering of directions in order to ensure that the diameter of the feasible region satisfies \(\diam(\Phi_t)\le 1+\beta_d\), where \(\beta_d\) is the constant appearing in the statement of Theorem~\ref{thm:robust-jung}. 

This preprocessing step is where we pay to control \(\diam(\Phi_t)\).
This cost is unavoidable in high dimensions: Main Theorem~\ref{thm:dim-bounds} shows that with a subexponential query budget one cannot uniformly bound the diameter. In particular, an exponential-in-\(d\) number of queries is necessary before the refinement phase.

\paragraph{Refinement.}
Assume that \(\diam(\Witness_t) \le 1+\beta\) at time \(t\). Theorem~\ref{thm:robust-jung} then applies to the witness core $\Witness_t$.
Hence, there exists a vertex set $\Delta := \{x_0,x_1,\ldots,x_d\}$ of a single regular simplex such that \emph{every} witness $W$ of $\Phi_t$ is $O_d(\beta)$-close to $\Delta$. In particular, since $\Witness_t$ is the union of such witnesses, it follows that
\[
\Witness_t \subseteq \Delta + B\!\bigl(O_d(\beta)\bigr).
\]
We then query the $\binom{d+1}{2}$ directions parallel to the edges of $\Delta$,
and denote the resulting feasible region by $\Phi_{\tnew}$ and its witness core
by $\Witness_{\tnew}$. Since $\Phi_{\tnew}\subseteq \Phi_t$, we also have
$\Witness_{\tnew}\subseteq \Witness_t$.

To show that \(\diam (\Witness_\tnew) \le 1+O_d(\beta^2)\), take any two points \(y_1, y_2 \in \Witness_\tnew\), and let $x_1$ and $x_2$ be the vertices of $\Delta$ that are $O_d(\beta)$-close to
$y_1$ and $y_2$, respectively.

We claim that $\|y_1-y_2\| \le 1+O_d(\beta^2)$. If $x_1=x_2$, then $y_1$ and $y_2$ are already $O_d(\beta)$-close. Otherwise, consider the queried direction $v:=\vvv{x_1x_2}$. Because $y_1$ and $y_2$ lie within $O_d(\beta)$ of the endpoints of the edge direction $v$, a straightforward geometric estimate shows that the segment $y_1y_2$ is tilted from $v$ by at most an $O_d(\beta)$ angle. Feasibility of \(y_1, y_2\) after querying \(v\) implies
\[
1 \ge \langle y_2-y_1, v\rangle =
\|y_1-y_2\|\cdot \cos \angle(\vvv{y_1y_2},v) =
\|y_1-y_2\|\cdot \cos O_d(\beta),
\]
and hence $\|y_1-y_2\|\le 1+O_d(\beta^2)$, using the Taylor expansion of cosine: $\cos \gamma \approx 1 - \frac12 \gamma^2$. Since this holds for any two points taken from \(\Witness_\tnew\), the bound on diameter follows. This concludes the refinement step.

Formally, the upper bound is proved in Section~\ref{sec:upper-bounds}.

\subsubsection{Lower bounds on excess error}

Although Theorem~\ref{thm:robust-jung} does not by itself yield the
lower bound, its formulation suggests an optimal adversary strategy. By
Theorem~\ref{thm:robust-jung}, after the preprocessing step, in the worst-case scenario for the reconstructor (i.e., when \(\rad(\Phi_T) > \Jung_d\)), all witnesses must cluster near the
vertex set of a single regular simplex.

This observation forms the core of the adversary's strategy. At each round \(t\), the adversary ensures that the set \(\Delta_t + B(\alpha_t)\) lies inside the feasible region, where \(\Delta_t\) is the vertex set of some regular simplex and
the radii \(\alpha_t\) can be chosen to satisfy a recurrence of the form \(\alpha_t \ge c\,\alpha_{t-1}^2\) for a constant \(c>0\). In particular, such a strategy forces the excess error at time \(t\) to be at least \(\alpha_t\), since
\[
\rad(\Phi_t) \ge \rad(\Delta_t + B(\alpha_t)) = \Jung_d + \alpha_t.
\]
Iterating this recurrence yields a doubly exponential rate.

This proof follows the same template as in~\cite{moranreconstruction}, and relies on
a careful rotation of \(\Delta_{t-1}\) to obtain the new simplex~\(\Delta_t\). We
defer the details to Section~\ref{sec:lower_bound}.

\subsubsection{Robust Jung Theorem: proof}
\label{sec:intro-proof-main-insight}

Finally, we explain the main ideas behind the proof of Robust Jung Theorem.
We begin by recalling its statement: 

\begin{rtheorem}{\ref{thm:robust-jung}}
Let $d \ge 1$ and let $S \subset \mathbb{R}^d$ be a compact set satisfying
\[
\diam(S) \le 1+\beta
\qquad\text{and}\qquad
\rad(S) \ge \Jung_d ,
\]
for some $0 \le \beta \le \beta_d$, where $\beta_d>0$ depends only on $d$.
Then, there exists a set of points
\[
\Delta = \{x_0,x_1,\dots,x_d\} \subset \mathbb{R}^d
\]
forming the vertex set of a regular simplex of edge length $1$ such that every witness \(W \subseteq S\) is \(C_d\beta\)-close to \(\Delta\),
for a constant $C_d>0$ depending only on $d$.
\end{rtheorem}

\smallskip

In the extremal case \(\beta=0\), the theorem yields essentially \emph{two} conclusions: first, every witness coincides with the vertex set of a regular simplex, and second, this witness is unique. The first conclusion follows directly from the proof of Jung’s theorem: after translating so that the center of the minimum enclosing ball is at the origin, equality in
Jung’s inequality implies $\sum_{i=0}^d x_i=0$, and a straightforward calculation then shows that the points form the vertex set of a regular simplex. 

The second statement is more delicate. Observe that any witness must lie on the boundary of the minimum enclosing ball of \(S\). Fix a witness
\(\{x_0,\ldots,x_d\}\), and suppose that \(y\) is another point of \(S\) lying on the same sphere. We show that \(y\) coincides with one of the \(x_i\). The argument relies on the fact that the vertices of a centered regular simplex are in isotropic position: for every \(y\in\R^d\),
\begin{equation}
\label{eq:isotropy}
\|y\|^2 \;=\; \tfrac12 \sum_{i=0}^d \langle y, x_i\rangle^2 .
\end{equation}
This identity captures the rigidity and symmetry of the regular simplex, and we use it repeatedly in the proof of  Robust Jung Theorem as well. Combining \Cref{eq:isotropy} with the assumption that \(y\) lies on the same sphere as the \(x_i\), a short linear-algebraic argument yields the desired conclusion. The formal proof appears in Lemma~\ref{lem:unique_regular}.

\medskip

The proof of Robust Jung Theorem (Theorem~\ref{thm:robust-jung}) follows the same outline, but in a robust form. As before, the main difficulty lies in the second step. First, we show that any witness is \(O_d(\beta)\)-close to the vertex set of some regular simplex. Second, we show that if the union of two regular simplices has diameter at most \(1+\beta\) (with \(\beta\) sufficiently small), then the simplices must be \(O_d(\beta)\)-close. 

\paragraph{Step 1: Any witness is close to a regular simplex.}

We first show, by a mild modification of the proof of Jung’s theorem, that every edge of a witness \(W\) has length \(1+O_d(\beta)\). We then construct a regular simplex \(\Delta\) by induction on the dimension \(d\). For \(n \le d\), let \(\Delta' \subset \R^{n-1}\) be an \((n-1)\)-dimensional regular simplex, and suppose a point \(y \in \R^n\) satisfies \(\dist(y,z)=1+O_d(\beta)\) for all vertices \(z \in \Delta'\). A quantitative inverse function theorem implies that \(y\) must be \(O_d(\beta)\)-close to some point \(x\) that completes \(\Delta'\) to an \(n\)-dimensional regular simplex. Iterating this construction yields the desired simplex \(\Delta\) and proves that \(W\) is \(O_d(\beta)\)-close to \(\Delta\). Step~1 is proved formally in Appendix~\ref{app:Jung_approx}.

\paragraph{Step 2: Clustering of witnesses.}

Step~1 allows us to replace each witness set \(W\) by the vertex set of a regular
simplex \(\Delta\) of edge length \(1\), and to exploit the rigidity of regular
simplices. Let \(\cX\) denote the family of all regular simplices of edge length
\(1\). The notion of \(r\)-closeness can be expressed on \(\cX\) via the Hausdorff
distance: for \(\Delta,\Delta' \in \cX\), the value \(\dist_H(\Delta,\Delta')\) is
the smallest~\(r\) such that \(\Delta\) and \(\Delta'\) are \(r\)-close. 

We also introduce a second function on $\cX\times\cX$, the
\emph{excess cross-diameter}
\[
D_{\diam}(\Delta,\Delta') \;:=\; \diam(\Delta \cup \Delta') - 1.
\]
This quantity is symmetric and vanishes if and only if \(\Delta=\Delta'\). We
show that, when restricted to pairs of sufficiently close simplices, \(D_{\diam}\) is bilipschitz equivalent to the Hausdorff distance. However, we do not know whether \(D_{\diam}\) satisfies the triangle inequality; in particular, local bilipschitz equivalence to a metric does not imply that \(D_{\diam}\) itself is a metric on \(\cX\).

The inequality \(\dist_H(\Delta,\Delta') \gtrsim D_{\diam}(\Delta,\Delta')\) is
immediate from the definitions. Indeed, if the Hausdorff distance between
\(\Delta\) and \(\Delta'\) is \(r\), then \(\Delta'\) lies inside the \(r\)-neighborhood of \(\Delta\). Consequently, the union \(\Delta \cup \Delta'\) cannot have diameter bigger than \(1+2r\). The main content of Step~2 is the reverse inequality. We show that there exists \(\beta_d>0\) such that for all distinct
\(\Delta,\Delta' \in \cX\) with \(D_{\diam}(\Delta,\Delta') \le \beta_d\),

\begin{equation}
\label{eq:bilipschitz}
c_d
\;\le\;
\frac{\dist_H(\Delta,\Delta')}{D_{\diam}(\Delta,\Delta')}
\;\le\;
C_d, 
\end{equation}
where constants \(c_d, C_d \) depend only on the dimension \(d\). 
To prove this, we work with the Euclidean motion group \(E(d)\), fix a reference
simplex \(\Delta\), and analyze motions \(T \in E(d)\) near the identity. We
show that sufficiently small motions, in the sense that they move every point
by at most \(\alpha_d\) for a dimension-dependent constant \(\alpha_d\),
satisfy~\eqref{eq:bilipschitz} for \(\Delta\) and \(T(\Delta)\). This is proved
in Lemma~\ref{lem:dist_diam_linear_in_speed} using Lie group--algebra
techniques.

Finally, a compactness argument shows that there exists \(\beta_d>0\) such that
if \(D_{\diam}(\Delta,\Delta')\) is at most \(\beta_d\), then there exists a
motion \(T \in E(d)\) which is \(\alpha_d\)-small and satisfies
\(\Delta' = T(\Delta)\).  Applying the local estimate above to this motion
concludes the bilipschitz comparison between \(\dist_H\) and \(D_{\diam}\), and
hence the clustering of witnesses. See Appendix~\ref{app:step_2}.

\subsection{Dependence on the dimension: Main Theorem~\ref{thm:dim-bounds}}

We briefly outline the proof of Main Theorem~\ref{thm:dim-bounds}. The theorem
shows a sharp transition in the query budget as a function of the dimension:
for subexponential budgets \(T(d)=2^{o(d)}\) the excess error diverges as
\(d\to\infty\), while for superexponential budgets \(T(d)=2^{\omega(d)}\) the
excess error vanishes.

The argument in both directions relies on the standard comparison between the
diameter and the Chebyshev radius of the feasible region:
\[
  \tfrac{1}{2}\,\operatorname{diam}(\Phi_T)
  \;\le\;
  \operatorname{rad}(\Phi_T)
  \;\le\;
  \Jung_d \cdot \operatorname{diam}(\Phi_T).
\]
\paragraph{Lower bound (subexponential budgets).}
Assume \(T(d)=2^{o(d)}\). We construct an adversary which ensures that the feasible region has large diameter. Our adversary is a continuous analogue of the one constructed by \cite{DinurN03}.

The adversary simply answers \(0\) to all queries. In order to analyze the feasible region, it would be convenient to take the noise rate to be $\delta = 2\sqrt{\ln T}$.
Let $X,Y$ be two independent standard Gaussians. We will show that with positive probability, both $X$ and $Y$ belong to the feasible region, and \(|X - Y| = \Omega(\sqrt{d})\).

Consequently, the diameter of the feasible region is $\Omega(\sqrt{d})$, hence so is its Chebyshev radius.

Let $v_t$ be the direction queried at the $t$'th step. Then $\langle v_t, X \rangle$ is a standard Gaussian, hence $\Pr[\langle v_t, X \rangle \ge \delta] \leq \frac{1}{T^2}$ by the standard Gaussian tail bound. Consequently, $X$ and $Y$ are both feasible with probability at least $1 - \frac{2}{T}$.

On the other hand, $X - Y$ is a $d$-dimensional Gaussian with variance $2$, and so its expected norm is $\Omega(\sqrt{d})$. Standard anti-concentration bounds show that this holds even conditioned on $X$ and $Y$ being both feasible, completing the proof.

\paragraph{Upper bound (superexponential budgets).}
When \(T(d)=2^{\omega(d)}\), the preprocessing step of querying a dense net of directions already suffices.
A superexponential query budget \(T(d)\) is enough to query a set \(S\) of directions such that any \(v \in S^{d-1}\) is \(o_d(1)\)-close to some direction in \(S\), which forces
\(\operatorname{diam}(\Phi_T)=2\delta+o_d(1)\). Applying Jung's theorem then gives \(\operatorname{rad}(\Phi_T) \le \Jung_d \cdot 2\delta+o_d(1)\), so the excess error vanishes. 

Both proofs appear in Section~\ref{sec:dim-bounds}.

\section{Model and Definitions}
\label{sec:formal-definitions}

We now give the formal definition of the linear reconstruction game, the
associated minimax optimal reconstruction error, and supporting lemmas.
Throughout this section we restrict attention to reconstructors with
point-valued outputs, i.e., we do not consider improper reconstructors.
The improper model is deferred to Section~\ref{sec:improper}. 

\begin{definition}[Linear reconstruction game and loss]
\label{def:interaction}
Fix an ambient space $\R^d$ and a noise level $\delta>0$.
The linear reconstruction game is played between an \emph{adversary} and a
\emph{reconstructor} as follows.

\begin{enumerate}[topsep=3pt,itemsep=1pt,parsep=0pt]
    \item At the beginning of the game, the adversary chooses a secret point
    $x^* \in \R^d$.

    \item The interaction lasts for $T$ rounds. In round $t\in[T]$:
    \begin{enumerate}
        \item The reconstructor chooses a query direction
        $v_t \in \R^d$ with $\|v_t\|_2=1$, possibly adaptively as a function of
        previous answers.
        \item The adversary returns a noisy response
        $r_t \in \R$ satisfying
        \[
          \bigl|r_t-\langle x^*,v_t\rangle\bigr|\le \delta.
        \]
    \end{enumerate}

    \item At the end of the interaction, the reconstructor outputs an estimate
    $\hat x_T\in\R^d$.
\end{enumerate}

We measure the error of the estimate $\hat x_T$ with respect to the secret $x^*$ by \(\ell_2\)-distance \(\|\hat x_T-x^*\|_2\).
\end{definition}

Recall that the \emph{optimal reconstruction error} is the smallest worst-case error that a reconstructor can guarantee after \(T\) rounds of interaction. Formally,

\begin{definition}[Optimal reconstruction error]
\label{def:optimal}
Fix a dimension \(d \ge 1\) and work in \(\R^d\).
Let \(\RC\) be the set of reconstruction strategies for \(T\) rounds.
We define the \emph{optimal reconstruction error} by
\[
  \OPT_d(T,\delta)
  := \inf_{\cR \in \RC}\ 
     \sup_{x^*\in\R^d}\ 
     \sup_{\cA \in \ADV(x^*)}
     \|\hat x_T - x^*\|_2 .
\]
Here \(\ADV(x^*)\) denotes the set of adversary response strategies consistent
with the secret point \(x^*\), and \(\hat x_T\) is the final output of \(\cR\).
All strategies are assumed deterministic.
\end{definition}

The first basic properties of the optimal reconstruction error are nonnegativity
and monotonicity in the number of rounds. The argument is standard and does not
use any properties of the linear reconstruction game beyond the fact that the
loss is nonnegative.

\begin{lemma}[Monotonicity and nonnegativity in $T$]\label{lem:opt-monotone}
For every $d\ge 1$ and noise level \(\delta>0\), the sequence
\[
T \mapsto \OPT_d(T, \delta)
\]
is nonincreasing and nonnegative.
\end{lemma}
\begin{proof}
Nonnegativity is immediate from the definitions since the loss is nonnegative, and taking a supremum over adversaries and an infimum over reconstructors preserves
nonnegativity, hence $\OPT_d(T, \delta)\ge 0$ for all $T$.

We next prove monotonicity. Fix $T\ge 0$ and $\alpha>0$.
By definition of $\OPT_d(T, \delta)$, there
exists a reconstructor $\cR$ that guarantees
worst-case error at most $\OPT_d(T, \delta)+\alpha$ after $T$ rounds.

Define a reconstructor $\cR'$ for $T+1$ rounds as follows.
$\cR'$ simulates $\cR$ for the first $T$ rounds, then issues an arbitrary
query at round $T+1$, ignores the response, and outputs exactly the same final
estimate as $\cR$ would after $T$ rounds.

For any adversary, the error of $\cR'$ coincides with the error of $\cR$.
Therefore, $\cR'$ guarantees worst-case error at most
$\OPT_d(T, \delta)+\alpha$ in $T+1$ rounds, implying
\[
\OPT_d(T+1, \delta) \le \OPT_d(T, \delta)+\alpha.
\]
Since this holds for every $\alpha>0$, we conclude that
$\OPT_d(T+1, \delta) \le \OPT_d(T, \delta)$.
\end{proof}

We are now ready to prove a simple but useful scaling property in the noise level.

\begin{lemma}[Scaling in the noise level]
\label{lem:scaling-proof}
For every $T \in \NN$ and $\delta>0$, the optimal error scales linearly with $\delta$:
\[
  \OPT_d(T,\delta) \;=\; \delta\,\OPT_d(T,1).
\]
\end{lemma}

\begin{proof}
Fix \(T\in\NN\) and \(\delta>0\). Since all strategies are deterministic, we can couple a game at noise level \(1\) with a game at noise level \(\delta\) by a simple rescaling.

Take any adversary strategy \(\cA\) for noise level~\(1\). It is determined by a secret point \(x^*\in\R^d\) and a deterministic rule for answering queries. Define a new adversary \(\cA_\delta\) for noise level \(\delta\) as follows. Its secret is \(\delta x^*\). On each query \(v_t\), it computes the answer that \(\cA\) would return to \(v_t\) at noise level~\(1\), and then returns \(\delta\) times this value.

Take any reconstructor strategy \(\cR\) for noise level \(1\). Define a new reconstructor \(\cR_\delta\) for noise level \(\delta\) as follows. It issues exactly the same queries \(v_t\) that \(\cR\) would issue. Whenever it receives an answer \(r_t\), it rescales it to \(r_t/\delta\) and feeds this value to \(\cR\). In the end of the interaction if \(\cR\) outputs \(\hat x_\cR\), then \(\cR_\delta\) outputs \(\delta\hat x_\cR\).

These transformations are invertible by applying the same construction with factor \(1/\delta\). Now run \(\cR\) against \(\cA\) at noise level \(1\), and run \(\cR_\delta\) against \(\cA_\delta\) at noise level \(\delta\). By construction, the query sequence is identical, and the answers in the second interaction are exactly \(\delta\) times the answers in the first. Therefore the final outputs are also scaled by \(\delta\).

This gives
\[
\|\delta\hat x_\cR-\delta x^*\|=\delta\|\hat x_\cR-x^*\|.
\]
Taking the supremum over adversaries and the infimum over reconstructors yields
\[
\OPT_d(T,\delta)=\delta\,\OPT_d(T,1),
\]
since the above transformations between noise levels \(1\) and \(\delta\) are bijections on both the adversary and reconstructor strategy sets.
\end{proof}

Finally, we recall the definition of the feasible region (see~\cref{eq:feasible}),
the core geometric object of the game.

\begin{definition}[Feasible region]
\label{eq:feasible-region}
Given a transcript $(v_1,r_1),\dots,(v_T,r_T)$ of query directions and responses
observed over $T$ rounds at noise level $\delta$, the feasible region after $T$
rounds is
\[
  \Phi_T
  :=
  \Bigl\{
  x\in\R^d :
  |\langle x,v_t\rangle-r_t|\le \delta
  \ \text{for all } t\in[T]
  \Bigr\}.
\]
\end{definition}

The feasible region records exactly the information available to the reconstructor after the interaction. Under this viewpoint, we can simplify the
adversary model without changing the optimal reconstruction error.

\begin{lemma}[A posteriori adversaries]
\label{lem:apost}
In the definition of \(\OPT_d(T,\delta)\), we may replace the usual adversary
(which fixes \(x^*\) in advance) by an adversary that only maintains
feasibility: it chooses replies \(r_1,\ldots,r_T\) online so that
\(\Phi_T\neq\varnothing\), and after observing the reconstructor's output it
selects any \(x^*\in\Phi_T\).
For deterministic reconstructors, these two models yield the same minimax value.
\end{lemma}
\begin{proof}
Any \emph{a priori} adversary that fixes a secret \(x^*\in\R^d\) in advance and
answers within noise \(\delta\) produces a feasible transcript, hence is a
special case of the feasibility-maintaining adversary.

Conversely, fix a deterministic reconstructor \(\cR\) and consider any feasible
transcript \((v_t,r_t)_{t=1}^T\) produced against \(\cR\), with feasible region
\(\Phi_T\neq\varnothing\). Choose an arbitrary \(x^*\in\Phi_T\). Then, by
definition of \(\Phi_T\),
\[
|r_t-\langle x^*,v_t\rangle|\le \delta \qquad \text{for all } t\in[T].
\]
Define an \emph{a priori} adversary with secret \(x^*\) as follows: on the
query sequence generated by \(\cR\), it returns the prescribed replies
\(r_1,\ldots,r_T\). If \(\cR\) ever deviates from this query sequence, the
adversary answers thereafter truthfully with \(\langle x^*,v\rangle\) on any
query direction \(v\in S^{d-1}\). This strategy always respects the noise
constraint, and it reproduces the same transcript against \(\cR\).

Therefore, for deterministic reconstructors, allowing an a priori versus an a
posteriori adversary does not change the optimal reconstruction error.
\end{proof}
Finally we now connect the reconstruction loss to a purely geometric quantity of the
feasible region. Recall that the \emph{Chebyshev radius} of a set \(S\subset \R^d\) is the radius
of its smallest enclosing Euclidean ball,
\begin{equation}
\label{eq:Chebyshev-radius}
\rad(S) := \inf_{c \in \R^d} \sup_{y \in S} \|c - y\|_2 .
\end{equation}
Any minimizer \(c\) is called a \emph{Chebyshev center} of \(S\).
(When \(S\) is bounded and nonempty, a Chebyshev center exists.)

\begin{lemma}[OPT via the feasible region]
\label{rem:loss-feasible}
Fix a transcript with nonempty feasible region \(\Phi_T\).
For any reconstructor output \(\hat x\in\R^d\),
\[
\sup_{x^*\in\Phi_T} \|\hat x - x^*\|_2 \;\ge\; \rad(\Phi_T).
\]
Moreover, equality holds when \(\hat x\) is a Chebyshev center of \(\Phi_T\). Consequently,
\[
\OPT_d(T,\delta)
=\inf_{\substack{\cR\ \text{a reconstructor strategy}\\ \text{at noise level }\delta}}\ 
  \sup_{\substack{\text{feasible transcripts}\\ \text{generated against }\cR}}
  \rad(\Phi_T).
\]
\end{lemma}

\begin{proof}
By Lemma~\ref{lem:apost}, after the transcript is fixed the adversary may choose
any \(x^*\in\Phi_T\). Thus the worst-case loss of an output \(\hat x\) equals
\(\sup_{y\in\Phi_T}\|\hat x-y\|_2\), and minimizing this quantity over \(\hat x\)
is exactly the definition of \(\rad(\Phi_T)\) in~\eqref{eq:Chebyshev-radius}.
\end{proof}

\section{Proofs}
\label{sec:proofs}

\subsection{Querying a Covering of Directions}

Before proving our main results, we introduce several technical tools that will be used repeatedly throughout the proofs.

We begin by defining a useful angular metric on the unit sphere.
Let \(S^{d-1} \subset \R^d\).  
The \emph{angular distance} on \(S^{d-1}\) is defined by
\begin{equation}
\label{eq:angular-metric}
    \rho(x,y) \;:=\; \angle(x,y) \;=\; \arccos\bigl(\langle x,y\rangle\bigr)
  \in [0,\pi],
  \qquad x,y \in S^{d-1}.
\end{equation}
Given \(\alpha>0\), a subset \(V \subset S^{d-1}\) is called an
\(\alpha\)-covering in \((S^{d-1},\rho)\) if for every \(v \in S^{d-1}\) there exists
\(u \in V\) such that
\[
  \rho(v,u) \le \alpha.
\]
\begin{definition}
\label{def:covering-number}
The (angular) \emph{covering number} \(\cN_\angle(S^{d-1},\alpha)\) is the minimal
cardinality of an \(\alpha\)-covering in \((S^{d-1},\rho)\).
Since the sphere \(S^{d-1}\) is compact, this covering number is finite for
every~\(\alpha>0\).
\end{definition}
In Appendix~\ref{app:covering-number} we prove the following useful estimate.

\begin{lemma}
\label{lem:covering-number-bounds}
For any \(0<\alpha<\frac{\pi}{2}\), the covering number \(\cN_\angle(S^{d-1},\alpha)\) satisfies the following bounds:
\[
  \sqrt d \,\alpha^{-(d-1)}
  \;<\;
  \cN_\angle(S^{d-1},\alpha)
  \;<\;
  2^{2d} \,\alpha^{-(d-1)}.
\]
\end{lemma}
The following lemma is a straightforward but useful observation that will be
used repeatedly in the proofs of the main theorems.

\begin{lemma}\label{lem:diam-from-cover}
Consider noise scale $\delta > 0$. Assume that the reconstructor queries an $\alpha$-covering of $S^{d-1}$ in the angular metric, for some $0 < \alpha < \pi/2$, that is, it makes at least  $\cN_\angle(S^{d-1},\alpha)$ queries (see Definition~\ref{def:covering-number} for the covering number).

Then, for any adversary the diameter of the resulting feasible region is bounded by
\[
\diam \ \Phi_T \;\le\; \frac{2\delta}{\cos \alpha},
\]
and hence
\[
\sup_{\text{adversaries}} \diam \Phi_T \;\le\; \frac{2\delta}{\cos \alpha}.
\]
\end{lemma}

\begin{proof}
    Let $\{v_1,\dots,v_T\} \subset S^{d-1}$ be an $\alpha$-covering of the unit sphere,
    that is, for every $v \in S^{d-1}$ there exists $v_i$ such that
    \[
    \langle v, v_i \rangle \;\ge\; \cos \alpha,
    \]
    equivalently, the angle between $v$ and $v_i$ is at most $\alpha$.
    Let $\Phi_T$ be the resulting feasible region, which depends on the adversary holding the secret point~$x^*$.
    Take arbitrary $x,y \in \Phi_T$ and set
    \[
    v := \frac{x-y}{\|x-y\|}.
    \]
    By the $\alpha$-covering property, there exists $v_i$ with
    $\langle v, v_i \rangle \ge \cos \alpha$.
    Since both $x$ and $y$ are consistent with the answer in direction $v_i$, we
    have
    \[
      |\langle x, v_i \rangle - r_i| \le \delta,
      \qquad
      |\langle y, v_i \rangle - r_i| \le \delta,
    \]
    and hence
    \[
      |\langle x, v_i \rangle - \langle y, v_i \rangle|
      \;\le\; 2\delta.
    \]
    On the other hand,
    \[
      \langle x, v_i \rangle - \langle y, v_i \rangle
      = \langle x-y, v_i \rangle
      = \|x-y\|\,\langle v, v_i \rangle
      \ge \|x-y\|\,\cos \alpha.
    \]
    Combining the two inequalities yields
    \[
      \|x-y\|
      \;\le\;
      \frac{2\delta}{\cos \alpha}.
    \]
    Since $x,y \in \Phi_T$ were arbitrary, we conclude that
    \begin{equation}\label{eq:diam-phiT}
      \diam(\Phi_T)
      \;\le\;
      \frac{2\delta}{\cos \alpha}.
    \end{equation}
\end{proof}

\subsection{Proof of Main Theorem~\ref{thm:asymptotic}}
\label{app:asymptotic}

In this section we determine the limiting optimal reconstruction error in the linear reconstruction game on $\R^d$ with normalized noise, as the number of interaction rounds $T$ tends to infinity. We will also need the following theorem, which links diameter and radius in Euclidean space. Recall Jung's theorem:

\begin{rtheorem}{\ref{thm:Jung}}
For every bounded set \(S \subseteq \R^d\) with diameter \(D\), its Chebyshev
radius satisfies
\[
  \rad(S) \le \sqrt{\tfrac{d}{2(d+1)}}\,D;
\]
see, e.g., \cite{Blumenthal1953}.
The factor \(\sqrt{\tfrac{d}{2(d+1)}}\) is known as \emph{Jung's constant}; we
denote it by \(\Jung_d\).
\end{rtheorem}

Now we are ready to prove Main Theorem~\ref{thm:asymptotic}. Recall it's statement: 
\bigskip
\begin{restatedtheorem}{\ref{thm:asymptotic}}
Consider the linear reconstruction game on $\R^d$, where $d\ge 1$, with
noise level \(\delta >0\). Let $\OPT_d(T, \delta)$ denote the optimal reconstruction errors after $T$ rounds, as defined in Definition~\ref{def:optimal}. 
Then the limit in \(T\) exists and satisfies
\[
  \OPT_d(\infty, \delta) := \lim_{T \to \infty} \OPT_d(T, \delta) = \sqrt{\frac{2d}{d+1}} \delta.
\]
\end{restatedtheorem}

\begin{proof}
Existence of the limit follows from the monotonicity and nonnegativity of \(T \mapsto \OPT_d(T,\delta)\) (Lemma~\ref{lem:opt-monotone}). To identify its value, we give an upper bound and a matching lower bound.

\paragraph{Upper bound.}
The main ingredient is Lemma~\ref{lem:diam-from-cover}, which provides a
reconstruction strategy based on querying a sufficiently fine angular covering
of the sphere.

Let $n\in\NN$ and let $T_n := \cN_\angle(S^{d-1},1/n)$ be the size of a
$1/n$-angular covering of $S^{d-1}$ (see Definition~\ref{def:covering-number}).
Applying Lemma~\ref{lem:diam-from-cover}, after $T_n$ queries the feasible region
$\Phi_{T_n}$ satisfies
\[
  \diam(\Phi_{T_d}) \;\le\; \frac{2}{\cos(1/n)}.
\]
Jung's theorem (Theorem~\ref{thm:Jung}) states that any subset of $\R^d$ of diameter $D$ has Chebyshev radius at most
$\sqrt{\tfrac{2d}{d+1}}\cdot D/2$.
Applying this to $\Phi_{T_n}$ gives
\[
  \OPT_d(T_n, \delta)
  \;\le\;
  \sqrt{\frac{2d}{d+1}}\cdot \frac{\delta}{\cos(1/n)}.
\]
Letting $n\to\infty$ yields
\[
  \OPT_d(\infty) \;\le\; \sqrt{\frac{2d}{d+1}} \delta.
\]
\paragraph{Lower bound.}
We now prove matching lower bound. Let \(S \subset \R^d\) be the vertex set of a regular simplex of diameter
\(2\delta\). When the reconstructor issues a query direction \(v_t \in S^{d-1}\),
the adversary projects \(S\) onto the line spanned by \(v_t\) and returns the
midpoint of the resulting interval. Formally, for \(v \in S^{d-1}\) define
\[
m_v(S) := \inf_{s \in S} \langle v, s \rangle,
\qquad
M_v(S) := \sup_{s \in S} \langle v, s \rangle,
\]
and set
\[
r_t := \frac{M_{v_t}(S) + m_{v_t}(S)}{2}.
\]
Then for every \(x \in S\),
\[
  |\langle x, v_t\rangle - r_t| \le \delta,
\]
and hence \(S \subseteq \Phi_t\) for all \(t\).

Since the adversary ensures \(S \subseteq \Phi_T\) for every transcript, we have
\[
  \rad(\Phi_T) \ge \rad(S) = 2\Jung_d\,\delta .
\]
By Lemma~\ref{rem:loss-feasible}, the worst-case loss of a reconstructor for feasible region \(\Phi_T\) equals \(\rad(\Phi_T)\). Therefore,
for every \(T\),
\[
  \OPT_d(T,\delta) \ge 2\Jung_d\,\delta
  = \sqrt{\frac{2d}{d+1}}\,\delta,
\]
and hence
\[
  \OPT_d(\infty,\delta) \ge \sqrt{\tfrac{2d}{d+1}}\,\delta.
\]
\smallskip
Combining the upper and lower bounds completes the proof.
\end{proof}

\subsection{Proof of Main Theorem~\ref{thm:excess-bounds}}

The goal of this section is to establish the doubly-exponential convergence rate
of the excess error. We restate the main theorem with explicit constants (rather
than the \(\Theta_d(\cdot)\) notation used in the introduction).

Recall that the excess error is defined by
\begin{equation}
\label{eq:excess-error}
\Excess_d(T,\delta) := \OPT_d(T,\delta) - \OPT_d(\infty,\delta),
\end{equation}
where \(\OPT_d(T,\delta)\) is the optimal reconstruction error (see
Definition~\ref{def:optimal}) and \(\OPT_d(\infty,\delta)\) is the
limiting value of the optimal error (see Main Theorem~\ref{thm:asymptotic}).

\bigskip 
\begin{restatedtheorem}{\ref{thm:excess-bounds}}
There exist constants \(a_d, A_d>0\) and a threshold \(T_d'\), depending only on
the dimension \(d\), such that for all \(T > T_d'\),
\[
  \delta \cdot 2^{-2^{A_d \cdot T}} \le \Excess_d(T,\delta) \le \delta \cdot 2^{-2^{a_d \cdot T}}.
\]
\end{restatedtheorem}
\bigskip

For normalization, we assume \(\delta=\tfrac12\).
By the linear scaling in the noise level (Lemma~\ref{lem:scaling-proof}),
\[
\OPT_d(T,\delta)=\delta\,\OPT_d(T,1)
\qquad\text{and}\qquad
\OPT_d(\infty,\delta)=\delta\,\OPT_d(\infty,1).
\]
Therefore,
\[
\OPT_d(T,\tfrac12)-\OPT_d(\infty,\tfrac12)
=\frac{1}{2\delta}\Bigl(\OPT_d(T,\delta)-\OPT_d(\infty,\delta)\Bigr),
\]
and it suffices to prove doubly exponential upper and lower bounds on \(\OPT_d(T,\tfrac12)-\OPT_d(\infty,\tfrac12)\).

We begin with the upper bound and then turn to the lower bound.

\subsubsection{Upper bound in Main Theorem~\ref{thm:excess-bounds}}
\label{sec:upper-bounds}

To provide an upper bound on excess error, we recall the core concept introduced in Section~\ref{sec:proof-overview}. 

\begin{definition}[Witness]
Let \(S\subset\R^d\).
A \emph{witness} for \(S\) is a subset \(W\subseteq S\) of cardinality \(|W|=d+1\) such that \(\rad(W)\;\ge\;\Jung_d\), where \(\Jung_d=\sqrt{\frac{d}{2(d+1)}}\) is the Jung constant (see Theorem~\ref{thm:Jung}).
If \(\rad(S)\ge \Jung_d\), then \(S\) admits a witness by
Lemma~\ref{lem:minimum-witness-cardinality}.
\end{definition}

We can then introduce a witness core of a subset \(S \subset \R^d\). 

\begin{equation}
\label{eq:witness-core}
\Witness(S) \;:=\; \bigcup_{\substack{W \subset S:\\
W \text{ is a witness } }} W. 
\end{equation}

In Appendix~\ref{app:Jung_approx} we prove the main technical tool used in the upper bound, namely the \emph{Robust Jung Theorem}. Using the notion of the witness core, it can be stated in the following form, which is more convenient for our proof.

\bigskip 
\begin{rtheorem}{\ref{thm:robust-jung}}
Let \(d \ge 1\) and let \(S \subset \R^d\) be a compact set satisfying
\[
\diam(\Witness(S)) \le 1+\beta
\qquad\text{and}\qquad
\rad(S) \ge \Jung_d,
\]
for some \(0 \le \beta \le \beta_d\), where \(\beta_d>0\) depends only on \(d\).
Then there exists a set of points
\[
\Delta = \{x_0,x_1,\dots,x_d\} \subset \R^d
\]
forming the vertex set of a regular simplex of edge length \(1\) such that
\[
\Witness(S) \subset \Delta + B(C_d \beta),
\]
where \(\Delta + B(r)\) denotes the Minkowski sum of \(\Delta\) with the
Euclidean ball of radius \(r>0\).
\end{rtheorem}
\bigskip 

Now we are ready to present the proof of Main Theorem~\ref{thm:excess-bounds} (Upper bounds).

\begin{restatedtheorem}{\ref{thm:excess-bounds}} (Upper bound)
There exist constants \(a_d>0\) and a threshold \(T_d'\), depending only on
the dimension \(d\), such that for all \(T > T_d'\),
\[
\Excess_d(T,\tfrac12) \le 2^{-2^{a_d \cdot T}}.
\]
\end{restatedtheorem}
\begin{proof}
Denote by \(\beta_d>0\) the constant from the Robust Jung Theorem
(Theorem~\ref{thm:robust-jung}), and define
\[
\alpha_d \;:=\; \arccos\!\left(\frac{1}{1+\beta_d}\right).
\]
Denote by \(T_0 := N(S^{d-1}, \alpha_d)\) the covering number of the sphere \(S^{d-1}\) with angular radius \(\alpha_d\) (see Definition~\ref{def:covering-number}). 

We construct, by induction on \(t\), a reconstruction strategy with the following
guarantee. After
\[
T_0 + \binom{d+1}{2}\cdot t
\]
queries, the witness core at time \(T_0+\binom{d+1}{2}t\) (which for convenience
we denote by \(\Witness_t\)) satisfies
\[
\Witness_t \subseteq \Delta_t + B(\gamma_t),
\]
where \(\Delta_t\) is the vertex set of a regular simplex of edge length \(1\),
and \((\gamma_t)_{t\ge 0}\) is a sequence obeying the recursion
\[
\gamma_0=C_d \cdot \beta_d,
\qquad
\gamma_t = A_d \cdot \gamma_{t-1}^2,
\]
for a constant \(A_d>0\) depending only on the dimension \(d\) such that \(C_d \cdot \beta_d < 1/A_d\). Here \(C_d\) is the constant from the Robust Jung Theorem (Theorem~\ref{thm:robust-jung}).

Since
\[
\rad(\Phi_t) \;=\; \rad(\Witness_t)
\;\le\; \rad\!\bigl(\Delta_t + B(\gamma_t)\bigr)
\;=\; \rad(\Delta_t)+\gamma_t
\;=\; \Jung_d+\gamma_t,
\]
this claim immediately yields a doubly exponential upper bound on the
convergence rate of \(\Excess_d(T,\delta)\).

\paragraph{Induction base.} The base case essentially follows from Lemma~\ref{lem:diam-from-cover} which states that whenever one queries an \(\alpha_d\)-covering of a sphere \(S^{d-1}\) with respect to the angular metric (see \Cref{eq:angular-metric}), one has
\[
\diam \Phi_{T_0} \le 2/\cos \alpha_d. 
\]
Since \(\Witness_{T_0}\subseteq \Phi_{T_0}\), we have
\(\diam(\Witness_{T_0}) \le \diam(\Phi_{T_0})\).
Thus Theorem~\ref{thm:robust-jung} applies to \(S=\Phi_{T_0}\), and yields a
regular simplex
\(\Delta_0 := \{x_0,x_1,\ldots,x_d\}\) of edge length \(1\) such that
\[
\Witness_{T_0} \subseteq \Delta_0 + B(C_d\,\beta_d).
\]
This establishes the base case.

\paragraph{Induction step.}
Assume the induction hypothesis at step \(t\):
\[
\Witness_t \subseteq \Delta_t + B(\gamma_t),
\]
where \(\Delta_t=\{x_0,x_1,\ldots,x_d\}\). We query the directions
\(v_{ij}:=x_i-x_j\) for all \(0\le i<j\le d\).

We claim that
\[
\diam(\Witness_{t+1}) \le 1+4\gamma_t^2.
\]
First note that \(\Witness_{t+1}\subseteq \Witness_t\).
Hence any two points in the new witness core lie within distance \(\gamma_t\) of
vertices of \(\Delta_t\).
Let \(y_i \in B(x_i,\gamma_t)\) and \(y_j \in B(x_j,\gamma_t)\), and assume that
\(x_i\neq x_j\).

Observe that
\[
\sin \bigl(y_i - y_j,\, x_i - x_j\bigr) \;\le\; \frac{2\gamma_t}{1 - 2\gamma_t},
\]
which follows from the sine theorem applied to the triangle with vertices
\(A = 0\), \(B = x_i - x_j\), and \(C = y_i - y_j\). Indeed,
\[
|AB| = 1,
\qquad
|BC| \le 2\gamma_t.
\]
By the sine theorem,
\[
\frac{|AB|}{\sin \angle BCA}
= \frac{|BC|}{\sin \angle BAC}
\quad\Rightarrow\quad
\sin \angle BAC
= \frac{|BC|}{|AB|}\cdot \sin \angle BCA
\;\le\;
2\gamma_t .
\]
Since \(\sin \angle BAC = \sin (y_i - y_j,\, x_i - x_j)\), we obtain
\[
\sin (y_i - y_j,\, x_i - x_j)
\;\le\;
2\gamma_t.
\]
Therefore, since we queried \(v_{ij}=x_i-x_j\) and both points \(y_i,y_j\) remain
feasible,
\[
1 \ge \langle y_i - y_j,\, v_{ij} \rangle
= \|y_i - y_j\|_2\sqrt{1 - \sin^2 \bigl(y_i - y_j,\, x_i - x_j\bigr)}
\ge \|y_i - y_j\|_2\cdot \sqrt{1 - 4\gamma_t^2}.
\]
Whenever \(\gamma_t \le \tfrac{1}{4}\), we have
\[
\frac{1}{\sqrt{1 - 4\gamma_t^2}} \le 1 + 4\gamma_t^2,
\]
and hence
\[
\|y_i - y_j\| \le 1 + 4\gamma_t^2.
\]
Since \(y_i,y_j\) were arbitrary, it follows that
\[
\diam(\Witness_{t+1}) \le 1 + 4\gamma_t^2.
\]
Finally, we apply the Robust Jung Theorem (Theorem~\ref{thm:robust-jung})
to \(\Witness_{t+1}\). It follows that there exists a regular simplex
\(\Delta_{t+1}\) such that
\[
\Witness_{t+1} \subseteq \Delta_{t+1} + B(4C_d\,\gamma_t^2).
\]
Setting \(\gamma_{t+1} := 4C_d\,\gamma_t^2\) completes the induction step.
\end{proof}

\subsubsection{Lower bound in Main Theorem~\ref{thm:excess-bounds}}
\label{sec:lower_bound}

We begin with a lemma describing how the adversary can answer the reconstructor’s queries so as to maintain \(\Delta + B(\alpha)\) inside the feasible region.
Here \(\Delta\) denotes the vertex set of a regular simplex in \(\R^d\), and \(\Delta + B(\alpha)\) is the Minkowski sum
\[
\Delta + B(\alpha)
\;=\;
\{\,x+y : x\in \Delta,\ \|y\|_2\le \alpha\,\}.
\]
We then use this lemma to derive lower bounds on the excess error.

\begin{lemma}[Criteria for the good neighborhood of regular simplex]
\label{lem:linear_criteria_good}
Fix \(0<\alpha<\tfrac{1}{4}\), and let \(\Delta = x_0x_1\ldots x_d\) be a regular simplex with edge length \(1\).
For a given direction \(v\in\R^d\), without loss of generality choose the edge \(x_0x_1\) of \(\Delta\) so that
\[
\langle v, x_0 - x_1\rangle \;=\; \max_{i,j}\,\langle v, x_i - x_j\rangle .
\]
If \(\langle v, x_0 - x_1\rangle \le 1 - 2 \alpha\), then, with \(r := \langle x_0 + x_1, v\rangle/2\), one has
\[
\Delta + B(\alpha) \subset \Phi(\{v,r\}).
\]
\end{lemma}

\begin{proof}

First, observe that for any vertex \(x_i\) one has
\[
\langle x_1 - x_0, v\rangle \;\le\; \langle x_0 + x_1 - 2x_i, v\rangle \;\le\; \langle x_0 - x_1, v\rangle.
\]
Indeed, suppose for contradiction that 
\(\langle x_0 + x_1 - 2x_i, v\rangle > \langle x_0 - x_1, v\rangle\).
Then
\(\langle x_1 - x_i, v\rangle > 0\), and hence
\(\langle x_0 - x_i, v\rangle > \langle x_0 - x_1, v\rangle\),
contradicting the maximality of the latter.
The other inequality is proved analogously.

Now assume \(x \in B(x_i,\alpha)\). We claim that \(x \in \Phi(v,r)\).
Write \(x = x_i + e\) with \(\|e\|_2 \le \alpha\). Then
\[
\left\langle \tfrac{x_0+x_1}{2} - x,\, v \right\rangle
   \;=\; \left\langle \tfrac{x_0+x_1}{2} - x_i,\, v \right\rangle - \langle e, v\rangle
   \;\le\; \tfrac12\langle x_0 - x_1, v\rangle + \alpha 
   \;\le\; \tfrac{1}{2},
\]
where the second inequality uses the bound \(\langle x_0+x_1-2x_i, v\rangle \le \langle x_0 - x_1, v\rangle\) shown above (and \(|\langle e,v\rangle|\le \alpha\)), and the third uses the hypothesis \(\langle x_0 - x_1, v\rangle \le 1 - 2\alpha\).
The opposite inequality,
\(\langle (x_0+x_1)/2 - x, v\rangle \ge -\tfrac{1}{2}\), is proved in the same way.
\end{proof}

Now we are ready to prove the lower bound. Recall the statement of the theorem: 

\begin{restatedtheorem}{\ref{thm:excess-bounds}} (Lower bound)
There exist constant \(A_d>0\) depending only on
the dimension \(d\), such that for all $T\ge d$
\[
\Excess_d(T,\tfrac12) \ge 2^{-2^{A_d \cdot T}}.
\]
\end{restatedtheorem}
\begin{proof}
It suffices to show that there exists a universal constant \(M>0\) such that, for all sufficiently small \(\alpha>0\), for every regular simplex \(\Delta\) of edge length \(1\) and every direction \(v\), there exist a response \(r\) and a simplex \(\Delta'\) with
\[
    \Delta' +B(\alpha') \subset \Delta + B(\alpha)
    \qquad\text{and}\qquad
    \Delta' + B(\alpha') \subset \Phi(v,r),
\]
where \(\alpha' = M\alpha^{2}\). We claim that the choice \(M=\tfrac{1}{17}\) is sufficient.

Without loss of generality, assume that the simplex is
\[
    \Delta = \{0, x_1, \ldots, x_d\},
    \qquad
    \langle v, x_1\rangle = \max_{i,j}\langle v, x_i - x_j\rangle.
\]
If \(\langle v, x_1\rangle \le 1 - 2\alpha^2/17\), then Lemma~\ref{lem:linear_criteria_good} directly produces a response \(r\) with \(\Delta + B(\alpha^2/17) \subset \Phi(v,r)\), and we are done with \(\Delta' = \Delta\). Suppose instead that
\[
\langle v, x_1\rangle > 1 - 2\alpha^2/17 \ge 1 - \alpha^2/8 + O(\alpha^4) = \cos(\alpha/2).
\]
In this case we rotate \(\Delta\) slightly so that Lemma~\ref{lem:linear_criteria_good} becomes applicable to the rotated simplex. The rotation is provided by Lemma~\ref{lem:simplex_rotation} (Appendix~\ref{app:rotation_lemmas}): given any direction \(v\) and any angle \(\theta \le \pi/18\), it constructs an isometry \(R_\theta\) such that the rotated simplex \(\Delta' := R_\theta \Delta\) satisfies \(\langle v, x_1' \rangle \le \cos\theta\) while still attaining the maximum at the edge \(0x_1'\).

Applying Lemma~\ref{lem:simplex_rotation} with \(\theta = \alpha/2\) produces a rotated simplex \(\Delta'\) with
\[
    \langle v, x_1'\rangle \le \cos(\alpha/2) \le 1 - 2\alpha^2/17,
    \qquad
    \langle v, x_1'\rangle = \max_{i,j}\langle v, x_i' - x_j'\rangle.
\]
Lemma~\ref{lem:linear_criteria_good} then yields a response \(r\) such that \(\Delta' + B(\alpha^2/17) \subset \Phi(v,r)\).

It remains to verify that \(\Delta' + B(\alpha^2/17) \subset \Delta + B(\alpha)\). For this we use Lemma~\ref{lem:rotating_neighborhood} (Appendix~\ref{app:rotation_lemmas}), which reduces this inclusion to a bound on the single quantity \(\|x_1 - x_1'\|_2\): namely, it suffices to check that
\(\|x_1 - x_1'\|_2 \le \alpha - \alpha^2/17\). Indeed,
\[
    \|x_1 - x_1'\|_2^2 = 2(1-\cos(\alpha/2)) = \alpha^2/8 + O(\alpha^4) \le (\alpha - \tfrac{1}{17}\alpha^2)^2
\]
for all sufficiently small \(\alpha\). 

Iterating this construction over \(T\) rounds yields \(\Delta_T + B(\alpha_T) \subset \Phi_T\) with \(\alpha_T \ge 2^{-2^{A_d T}}\), and hence \(\Excess_d(T, \tfrac12) \ge \rad(\Phi_T) - \Jung_d \ge \alpha_T\). This concludes the proof.
\end{proof}

\subsection{Proof of Main Theorem~\ref{thm:dim-bounds}}
\label{sec:dim-bounds}

The goal of this section is to quantify how the query budget \(T(d)\) required
to control the excess error depends on the dimension \(d\).

Recall that the excess error is
\[
  \Excess_d(T,\delta) := \OPT_d(T,\delta) - \OPT_d(\infty,\delta),
\]
where \(\OPT_d(\infty,\delta)\) denotes the limiting optimal error from
Main Theorem~\ref{thm:asymptotic}.

\smallskip
We use standard asymptotic notation as \(d\to\infty\). For \(f,g:\NN\to\R_+\),
we write \(f(d)=o(g(d))\) if \(\frac{f(d)}{g(d)}\to 0\), and
\(f(d)=\omega(g(d))\) if \(\frac{f(d)}{g(d)}\to +\infty\).

\bigskip 
\begin{restatedtheorem}{\ref{thm:dim-bounds}}
Let \(T:\NN\to\NN\) be a query budget as a function of the dimension \(d\).
\begin{enumerate}
  \item If \(\ln T(d)=o(d)\) for all sufficiently large \(d\),
  then
  \[
    \Excess_d(T(d),\delta) \xrightarrow{d\to\infty} +\infty.
  \]
  \item If \(\ln T(d)=\omega(d)\) for all sufficiently large \(d\), then
  \[
    \Excess_d(T(d),\delta) \xrightarrow{d\to\infty} 0.
  \]
\end{enumerate}
\end{restatedtheorem}
\begin{proof}
We treat the two regimes separately. The subexponential regime relies on a
Gaussian construction reminiscent of \cite{DinurN03}; the superexponential
regime reduces to the covering-number bound of
Lemma~\ref{lem:covering-number-bounds}.

\paragraph{Proof of Item 1 (subexponential budgets).}

By the scaling property (Lemma~\ref{lem:scaling-proof}), we have $\OPT_d(T, \delta) = \delta \OPT_d(T, 1)$. Thus, it suffices to show that
\[
  \OPT_d(T_d, 1) - \OPT_d(\infty,1) \xrightarrow{d \to \infty} +\infty.
\]
From Main Theorem~\ref{thm:asymptotic}, the asymptotic optimal error is bounded by $\OPT_d(\infty,1) = \sqrt{\frac{2d}{d+1}} \le \sqrt{2}$. Therefore, we need only show that
\[
  \OPT_d(T_d, 1) \xrightarrow{d \to \infty} +\infty. 
\]
To this end, let $\delta_d := 2\sqrt{\ln T_d}$. We will prove that for all sufficiently large $d$,
\begin{equation}
\label{eq:opt-lb-final}
\OPT_d(T_d,\delta_d) \ge \Jung_d \cdot \sqrt{d}.
\end{equation}
Assuming \eqref{eq:opt-lb-final}, the scaling lemma implies
\[
\OPT_d(T_d,1) = \frac{\OPT_d(T_d,\delta_d)}{\delta_d} \ge \frac{\Jung_d \cdot \sqrt{d}}{\delta_d} = \Jung_d \cdot \sqrt{\frac{d}{4\ln T_d}}.
\]
Since $\ln T_d = o(d)$, the right-hand side tends to $+\infty$ as $d \to \infty$.

It remains to establish \eqref{eq:opt-lb-final}.
Consider the a posteriori adversary that answers \(0\) to every query.
For a transcript with query directions \(v_1,\ldots,v_{T_d}\in S^{d-1}\),
the feasible region is
\[
\Phi_{T_d}
:= \bigcap_{t=1}^{T_d}\Bigl\{x\in\R^d:\ |\langle v_t,x\rangle|\le \delta_d\Bigr\}.
\]
Since the improper error is at least half the diameter of the feasible region,
it suffices to show that \(\diam(\Phi_{T_d})\ge \sqrt d\) for all large enough \(d\).

Let \(X,Y\sim \cN(0,I_d)\) be independent.
Using the standard Gaussian tail bound
\[
\Pr(|g|>t)\le e^{-t^2/2}\qquad(g\sim\cN(0,1),\ t\ge 0),
\]
we have, for any query direction \(v\in S^{d-1}\),
\[
\Pr\bigl(|\langle v,X\rangle|>\delta_d\bigr)\le e^{-\delta_d^2/2},
\]
and the same bound holds with \(X\) replaced by \(Y\).
A union bound over all \(T_d\) queries therefore yields
\[
\Pr(X\notin \Phi_{T_d}) + \Pr(Y\notin \Phi_{T_d})
\le 2T_d e^{-\delta_d^2/2}
= \frac{2}{T_d}.
\]
On the other hand, since \(X-Y\sim \cN(0,2I_d)\), we may write
\[
\|X-Y\|^2 \sim 2\,\chi_d^2,
\]
where \(\chi_d^2\) is chi-square with \(d\) degrees of freedom.
Thus
\[
\Pr(\|X-Y\|_2<\sqrt d)=\Pr(\chi_d^2<d/2).
\]
By a standard concentration inequality for the Chi-square distribution 
\[
\Pr(\chi_d^2<d/2)
\le \left(\frac{e^{1/2}}{2}\right)^{d/2},
\]
which decays exponentially in \(d\).

Combining the above bounds, we obtain
\[
\Pr\!\left[X\in\Phi_{T_d},\ Y\in\Phi_{T_d},\ \|X-Y\|\ge \sqrt d\right]
\;\ge\;
1 - \frac{2}{T_d} - \left(\frac{e^{1/2}}{2}\right)^{d/2}.
\]
We may assume \(T_d\ge 3\) for all sufficiently large \(d\), and hence the right-hand side is
strictly positive for all large \(d\).
Hence there exist \(x,y\in\Phi_{T_d}\) with \(\|x-y\|\ge \sqrt d\),
and therefore \(\diam(\Phi_{T_d})\ge \sqrt d\). Apply Jung's Theorem (Theorem~\ref{thm:Jung}). 
This proves \eqref{eq:opt-lb-final} and completes the proof.

\paragraph{Proof of Item 2 (superexponential budgets).}
The argument parallels the upper-bound proof of Main
Theorem~\ref{thm:asymptotic}; the only difference is that here we use the
explicit dimension-dependent covering bounds from
Lemma~\ref{lem:covering-number-bounds}.

Fix \(\alpha\in(0,\pi/2)\) and let \(T_\alpha := \cN(S^{d-1},\alpha)\).
By Lemma~\ref{lem:diam-from-cover}, after querying an \(\alpha\)-cover of
\(S^{d-1}\) the feasible region satisfies
\[
  \diam(\Phi_{T_\alpha}) \le \frac{2\delta}{\cos\alpha}.
\]
By Lemma~\ref{rem:loss-feasible}, the optimal worst-case error at time \(T_\alpha\)
is equal to the Chebyshev radius of the feasible region \(\Phi_{T_\alpha}\). By Jung's theorem
(\Cref{thm:Jung}),
\[
  \rad(\Phi_{T_\alpha}) \le \Jung_d \cdot \diam(\Phi_{T_\alpha})
  \le \Jung_d \cdot \frac{2\delta}{\cos\alpha},
\]
and hence
\[
  \OPT_d(T_\alpha,\delta) \le \Jung_d \cdot \frac{2\delta}{\cos\alpha}.
\]
Using \(\OPT_d(\infty,\delta)=2\delta\Jung_d\) (Main Theorem~\ref{thm:asymptotic}),
we obtain
\[
  \Excess_d(T_\alpha,\delta)
  = \OPT_d(T_\alpha,\delta) - \OPT_d(\infty,\delta)
  \le 2\delta\Jung_d\left(\frac{1}{\cos\alpha}-1\right).
\]
Now set
\[
  \beta_d := \frac{1}{\sqrt[d-1]{T(d)}}.
\]
Since \(\ln T(d)=\omega(d)\), we have \(\frac{\ln T(d)}{d-1}\to +\infty\), hence
\(\sqrt[d-1]{T(d)}=\exp\!\bigl(\tfrac{\ln T(d)}{d-1}\bigr)\to\infty\), and thus
\(\beta_d\to 0\).
By Lemma~\ref{lem:covering-number-bounds} (applied with \(\alpha=\beta_d\)),
we have \(T(d)\ge \cN(S^{d-1},\beta_d)=T_{\beta_d}\) for all sufficiently large
\(d\).
Using monotonicity in \(T\) (Lemma~\ref{lem:opt-monotone}), and therefore also of
\(\Excess_d(\cdot,\delta)\), we get
\[
  \Excess_d(T(d),\delta)\le \Excess_d(T_{\beta_d},\delta)
  \le 2\delta\Jung_d\left(\frac{1}{\cos\beta_d}-1\right).
\]
Finally, \(2\delta \cdot \Jung_d=\sqrt{\tfrac{2d}{d+1}}\,\delta<\sqrt2\,\delta\), and
since \(\beta_d\to 0\) and \(\cos\) is continuous with \(\cos 0=1\), the right-hand
side tends to \(0\). This proves the claim.
\end{proof}

\section{Improper reconstruction}
\label{sec:improper}

In this section, we analyze an \emph{improper} variant of the reconstruction game. In the improper reconstruction model, the reconstructor does not output a single estimate \(\hat x_T\in\R^d\). Instead, after observing the interaction transcript, it produces a prediction rule \(G:S^{d-1}\to\R\), whose goal is to predict future
answers \(\langle x^*,v\rangle\) for every direction \(v\in S^{d-1}\).

To define the optimal reconstruction error for the improper version, we must specify a loss for a prediction rule relative to the secret point \(x^*\).
In this work, we use a worst-case loss: after \(T\) rounds, imagine an additional ``evaluation round'' in which the adversary chooses the worst direction \(v \in S^{d-1}\).
The loss of a prediction rule \(\hat G_T : S^{d-1} \to \R\) is then
\[
  \sup_{v \in S^{d-1}}
  \big|\hat G_T(v) - \langle x^*, v\rangle\big|.
\]
\begin{definition}[Improper optimal reconstruction error]
\label{def:optimal-improper}
Fix a dimension \(d \ge 1\) and work in \(\R^d\).
Let \(\RC_{\improper}\) be the set of improper reconstruction strategies for \(T\)
rounds, i.e., strategies that after \(T\) rounds output a prediction rule
\(\hat G_T : S^{d-1}\to\R\).
We define the \emph{improper optimal reconstruction error} by
\[
  \OPT_d^{\improper}(T,\delta)
  := \inf_{\cR \in \RC_{\improper}}\ 
     \sup_{x^*\in\R^d}\ 
     \sup_{\cA \in \ADV(x^*)}\ 
     \sup_{v \in S^{d-1}}
     \bigl|\hat G_T(v) - \langle x^*, v \rangle\bigr|.
\]
Here \(\ADV(x^*)\) denotes the set of adversary response strategies consistent
with the secret point \(x^*\).
All strategies are assumed deterministic.
\end{definition}

The performance of an improper reconstructor is likewise governed by the
geometry of the feasible region, as in the usual reconstruction setting.
Recall that given a transcript $(v_1,r_1),\dots,(v_T,r_T)$ at noise level
$\delta>0$, the feasible region after $T$ rounds is
\[
  \Phi_T
  :=
  \Bigl\{
  x\in\R^d :
  |\langle x,v_t\rangle-r_t|\le \delta
  \ \text{for all } t\in[T]
  \Bigr\}.
\]
As in the usual reconstruction setting, we may take the adversary to be
a posteriori. Once the transcript is fixed, the adversary may choose any secret
point \(x^*\in \Phi_T\). The relevant geometric invariant is now the diameter:
given a feasible region \(\Phi_T\), the worst-case improper error is exactly
\(\diam(\Phi_T)/2\).

Indeed, for any direction \(v\), the set of possible values of
\(\langle x^*,v\rangle\) as \(x^*\) ranges over \(\Phi_T\) is the projection of
\(\Phi_T\) onto the line spanned by \(v\), hence it forms an interval.
If the diameter of \(\Phi_T\) is attained along a line parallel to \(v\), then
regardless of the prediction made by the reconstructor, the adversary can choose
a secret point \(x^*\in\Phi_T\) such that the prediction error in direction \(v\)
is at least \(\diam(\Phi_T)/2\).

Conversely, if the reconstructor predicts, for each direction \(v\), the
midpoint of this projection interval, then the prediction error in direction
\(v\) is at most \(\diam(\Phi_T)/2\). Therefore, in the improper model, the
worst-case error induced by a feasible region \(\Phi_T\) is exactly
\[
  \frac12\,\diam(\Phi_T).
\]
See \Cref{fig:improper-feasible} for a geometric illustration.

\begin{figure}[t]
\begin{center}
    \includegraphics[width=0.5\textwidth]{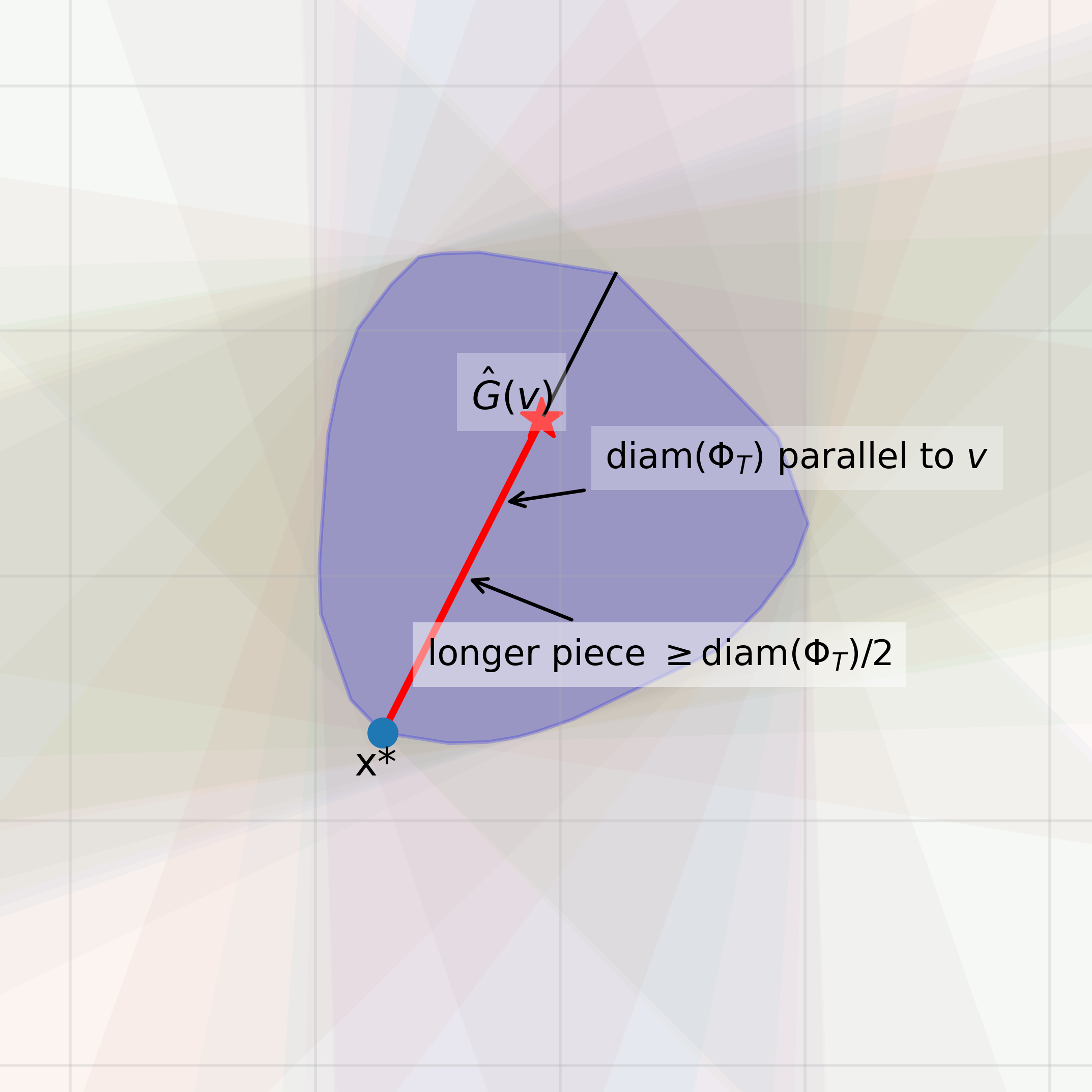}
    \caption{For any prediction rule \(\hat G\) there exists a secret point
    \(x^* \in \Phi_T\) and a direction \(v\) such that
    \(\bigl|\hat G(v) - \langle x^*, v\rangle\bigr| \ge \diam(\Phi_T)/2\).}
    \label{fig:improper-feasible}
\end{center}
\end{figure}

Formally, this observation is captured by the following lemma. We omit the
proof, as it is a direct formalization of the argument above.

\begin{lemma}[Improper OPT via the feasible region]
\label{lem:improper-loss-feasible}
Fix a transcript with a nonempty feasible region \(\Phi_T\).
For any reconstructor output \(\hat G : S^{d-1} \to \R\),
\[
\sup_{x^*\in\Phi_T}\ \sup_{v\in S^{d-1}}
\bigl|\hat G(v) - \langle x^*, v\rangle\bigr|
\;\ge\; \frac12\,\diam(\Phi_T).
\]
Moreover, equality holds for a prediction rule outputting the midpoint of the projection of \(\Phi_T\) onto the line spanned by queried direction \(v\).
Consequently,
\[
\OPT^\improper_d(T,\delta)
=\inf_{\cR}\ 
  \sup_{\substack{\text{feasible transcripts}\\ \text{generated against }\cR}}
  \frac12\,\diam(\Phi_T).
\]
\end{lemma}

Consequently, any upper (respectively, lower) bound on $\diam(\Phi_T)$ that holds for all transcripts after $T$ rounds immediately yields the corresponding upper (respectively, lower) bound on the improper minimax error, up to the factor $1/2$.

\subsection{Analogue of the Main Theorems for improper reconstruction}

In the improper model surprisingly the asymptotic optimal reconstruction error is strictly smaller than usual asymptotic reconstruction error:
\[
  \OPT_d^\improper(\infty, \delta) := \lim_{T\to\infty}\OPT_d^\improper(T,\delta)=\delta.
\]
Indeed, by Lemma~\ref{lem:improper-loss-feasible}, the loss is exactly $\tfrac12\diam(\Phi_T)$. Querying a sufficiently fine angular covering of the sphere ensures that the diameter of the feasible region is at most $(2+o(1))\delta$, and hence the improper reconstruction error is at most $(1+o(1))\delta$. Conversely, an adversary that answers $0$ to every query ensures that the diameter of the feasible region remains at least $2\delta$, forcing improper error at least $\delta$. 

It follows that the improper excess error is captured by exactly the same strategies that already yield the asymptotic improper error. This stands in contrast to the (usual) reconstruction model: there, controlling
the diameter alone is not sufficient, since the loss is governed by the Chebyshev radius, and one must additionally relate radius to diameter (via Jung-type arguments).

To control the diameter of the feasible region, the reconstructor must query a
sufficiently dense net of directions on the sphere. Quantitatively, this leads to a convergence rate dictated by the angular covering number of $S^{d-1}$.

\begin{theorem}[Improper excess error]
\label{thm:impr-excess-bounds}
Fix \(d \ge 1\). There exist constants $a_d, A_d>0$, depending only on \(d\),
and an integer \(T'_d\ge 1\) such that for all \(T \ge T'_d\),
\[
  \delta\bigl(1+ a_d\,T^{-2/(d-1)}\bigr)
  \;\le\;
  \OPT^\improper_d(T, \delta)
  \;\le\;
  \delta\bigl(1 + A_d\,T^{-2/(d-1)}\bigr).
\]
\end{theorem}

\begin{proof}
\paragraph{Upper bound.}
Let \(T \ge 16^{d-1}\). Then
\[
  \alpha \;:=\; \frac{8}{T^{1/(d-1)}} \;\le\; \frac{8}{16} \;=\; \frac12 \;<\; \frac{\pi}{4}.
\]
By Lemma~\ref{lem:covering-number-bounds}, the sphere admits an angular
\(\alpha\)-covering of size at most \(T\). Fix such a covering
\(\cV=\{v_1,\dots,v_T\}\subseteq S^{d-1}\).
Consider the reconstructor that, in rounds \(i=1,\dots,T\), queries \(v_i\).
Let \(r_i\) be the corresponding answers, satisfying
\(|r_i-\langle x^*,v_i\rangle|\le 1\), and let \(\Phi_T\) be the feasible
region induced by the transcript:
\[
\Phi_T \;=\;\Bigl\{x\in\R^d:\ |\langle x,v_i\rangle-r_i|\le 1\ \text{ for all }i\in[T]\Bigr\}.
\]
By Lemma~\ref{lem:diam-from-cover}, for any answers consistent with noise level \(1\)
we have
\[
  \diam(\Phi_T)\ \le\ \frac{2}{\cos\alpha}.
\]
Hence,
\[
  \sup_{\text{adversaries}}\ \diam(\Phi_T) \;\le\; \frac{2}{\cos\alpha}.
\]
Since \(\alpha\in[0,\pi/4]\), we may use the elementary inequality
\(\frac{1}{\cos\alpha}\le 1+2\alpha^2\) to obtain
\[
  \sup_{\text{adversaries}}\ \diam(\Phi_T)
  \;\le\; 2\bigl(1+2\alpha^2\bigr)
  \;=\; 2 + O(\alpha^2).
\]
Finally
\[
  \OPT_d^\improper(T)
  \;\le\;
  \frac12\cdot \sup_{\text{adversaries}}\diam(\Phi_T)
  \;\le\; \frac12\cdot \frac{2}{\cos\alpha}
  \;=\; \frac{1}{\cos\alpha}
  \;\le\; 1 + O(\alpha^2).
\]
Substituting \(\alpha = 8\,T^{-1/(d-1)}\) yields
\[
  \OPT_d^\improper(T) \;\le\; 1 + O\!\bigl(T^{-2/(d-1)}\bigr),
\]
as claimed.

\paragraph{Lower bound.}
Consider the adversary that answers \(y_t=0\) for every query \(v_t\).
Fix any (possibly adaptive) sequence of queries \(v_1,\dots,v_T\). 

Let \(\alpha>0\) be 
\[
\alpha:= \frac{1}{\sqrt[d-1]{T}}. 
\]
By Lemma~\ref{lem:covering-number-bounds} one has \(\cN_\angle(S^{d-1},\alpha) > 2T\) and hence, there exists \(v\in S^{d-1}\) such that for every \(i\in[T]\),
\(\rho(v,v_i)\ge \alpha\) and \(
  \rho(v,-v_i)\ge \alpha\), where \(\rho\) denotes the angular metric (see \Cref{eq:angular-metric}).
Equivalently, for all \(i\in[T]\), \[|\langle v, v_i\rangle| \;\le\; \cos\alpha.\]
Define two antipodal points
\[
  x := \frac{v}{\cos\alpha},
  \qquad
  y := -\frac{v}{\cos\alpha}.
\]
Then for every query \(v_i\) one has \( |\langle x, v_i\rangle| \le 1\), and similarly for \(y\). Hence \(x,y\in\Phi_T\), and therefore
\[
  \diam(\Phi_T)\;\ge\;\|x-y\|_2 \;=\; \frac{2}{\cos\alpha}.
\]
This implies
\[
  \OPT_d^\improper(T)\ \ge\ \frac{1}{\cos\alpha}.
\]
Expanding \(1/\cos\alpha = 1 + \Theta(\alpha^2)\) for small \(\alpha\), we obtain
\[
  \OPT_d^\improper(T) - 1 \;\ge\; \Theta(\alpha^2)
  \;=\; \Theta\!\bigl(T^{-2/(d-1)}\bigr),
\]
which gives the desired lower bound.
\end{proof}

\paragraph{Analogue of Main Theorem~\ref{thm:dim-bounds}}

Main Theorem~\ref{thm:dim-bounds} applies equally to the improper excess error
\(\OPT^\improper_d(T,\delta)-\delta\), since its proof proceeds by controlling
the diameter of the feasible region and then translating this control into a
bound on the loss.  In the usual reconstruction setting, this translation uses
Jung's theorem for the upper bound and the triangle inequality for the lower bound:
\[
\frac12\,\diam(\Phi_T)
\;\le\;
\rad(\Phi_T)
\;\le\;
\sqrt{\frac{d}{2(d+1)}}\,\diam(\Phi_T).
\]
In the improper setting, the loss is exactly \(\diam(\Phi_T)/2\), so no passage
through the Chebyshev radius is needed. Hence 
\begin{enumerate}[topsep=3pt,itemsep=1pt,parsep=0pt]
  \item If \(\ln T(d)=o(d)\) and \(T(d)\ge 3\) for all sufficiently large \(d\),
  then
  \[
    \OPT^\improper_d(T(d),\delta) \xrightarrow{d\to\infty} +\infty.
  \]
  \item If \(\ln T(d)=\omega(d)\) for all sufficiently large \(d\), then
  \[
    \OPT^\improper_d(T(d),\delta) \xrightarrow{d\to\infty} \delta.
  \]
\end{enumerate}

\section{Open directions}
\label{sec:open-directions}

We end with two variants of the reconstruction problem.

One natural variant is the nonadaptive linear reconstruction game, in which
all query directions are chosen before any answers are observed. The proof of
the doubly-exponential upper bound in Main Theorem~\ref{thm:excess-bounds}
relies on adaptivity: the refinement phase queries directions parallel to the
edges of a simplex extracted from the current feasible region. Related
reconstruction and query-release problems are studied extensively in the
privacy literature \citep{DinurN03,BlumLigettRoth13}, but those works primarily
concern discrete databases, subset/counting queries, and private release of
answers to query classes. To the best of our knowledge, the corresponding
minimax problem for nonadaptive approximate linear queries in \(\mathbb{R}^d\)
has not been characterized.

A second, closely related question is how the model should be extended when
some answers are allowed to violate the \(\delta\)-accuracy promise. In the
present model, every answer is consistent with the unknown point \(x^*\), and
hence \(x^*\) remains inside the feasible region throughout the interaction;
this invariant is central to our analysis, and once it fails the geometric
arguments used here do not apply directly. In related settings, several forms
of unreliable answers have been studied, including arbitrary corruptions in
reconstruction and decoding problems in the privacy literature
\citep{DworkMcSherryTalwar07}, and missing, faulty, or maliciously corrupted
answers in learning from membership queries
\citep{angluin1994randomly,angluin1997malicious}. Identifying the right robustness notions for the linear reconstruction game, and determining the corresponding optimal error and rates of convergence, remain open.

\acks{We thank the anonymous reviewers and the meta-reviewer for their careful reading
and helpful comments, which improved the presentation of the paper. We are
especially grateful to the meta-reviewer for suggesting the nonadaptive and
corrupted-answer variants discussed in Section~\ref{sec:open-directions}.

Shay Moran is a Robert J.\ Shillman Fellow and acknowledges support by Israel
PBC-VATAT, by the Technion Center for Machine Learning and Intelligent Systems
(MLIS), and by the European Union (ERC, GENERALIZATION, 101039692).

Yuval Filmus was supported by ISF grant 507/24.

Elizaveta Nesterova was supported by the Milner Fellowship for Ph.D. Students
and by the European Union (ERC, GENERALIZATION, 101039692).

Views and opinions expressed are however those of the author(s) only and do not
necessarily reflect those of the European Union or the European Research
Council Executive Agency. Neither the European Union nor the granting authority
can be held responsible for them.}

\bibliography{bib}
\appendix

\section{Robust Jung Theorem}
\label{app:Jung_approx}

In this appendix we formalize and prove the main technical tool used in the proof
of Main Theorem~\ref{thm:excess-bounds}, concerning sets whose diameter and
Chebyshev radius are close to extremal in Jung’s inequality.

We work here with a slightly different formulation than in the Technical
Overview; it will turn out that the two formulations are equivalent.
The main notion we use is the Hausdorff distance.

\subsection{Supporting definitions and lemmas}
\begin{definition}[Hausdorff distance]
For nonempty compact sets \(A,B \subset \R^d\), the \emph{Hausdorff distance} is
defined by
\[
  \dist_H(A,B)
  := \max\Bigl\{
      \sup_{a \in A} \inf_{b \in B} \|a-b\|_2,
      \sup_{b \in B} \inf_{a \in A} \|a-b\|_2
  \Bigr\}.
\]
\end{definition}

The fact that \(\dist_H\) is indeed a metric on the space of nonempty compact
subsets of \(\R^d\) can be found in, e.g.,~\cite{IllanesNadler1999}.

For a set \(S \subset \R^d\) and \(r>0\), we denote by
\[
  S + B(r) := \{x+y : x \in S,\ \|y\|_2 \le r\}
\]
the Minkowski sum of \(S\) with the Euclidean ball of radius \(r\).

It follows directly from the definitions that
\(\dist_H(S_1,S_2) \le r\) if and only if
\[
  S_1 \subseteq S_2 + B(r)
  \qquad\text{and}\qquad
  S_2 \subseteq S_1 + B(r),
\]
that is, \(S_1\) and \(S_2\) are \(r\)-close in the sense used in the Technical Overview.
When \(r=0\), this condition reduces to \(S_1 = S_2\).

\paragraph{Chebyshev radius.}
For a bounded set \(S \subset \R^d\), its \emph{Chebyshev radius} is
\[
  \rad(S) := \inf_{x \in \R^d} \sup_{y \in S} \|x-y\|_2.
\]
Any point attaining the infimum is called a \emph{Chebyshev center} of \(S\).

Recall Jung's theorem, for a proof see \cite[Theorem~3.3]{gruber2007convex}.

\begin{theorem*}[Jung's theorem]
Let \(S \subseteq \R^d\) be a bounded set with diameter \(D := \diam(S)\).
Then its Chebyshev radius satisfies
\[
  \rad(S) \le \sqrt{\frac{d}{2(d+1)}}\, D .
\]
The factor \(\sqrt{\frac{d}{2(d+1)}}\) is called \emph{Jung's constant}; we
denote it by \(\Jung_d\).
\end{theorem*}

\smallskip

Recall that a \emph{witness} of \(S\) is a subset \(W \subseteq S\) consisting of \(d+1\) points such that 
\[
  \rad(W) \ge \Jung_d.
\]
The following lemma is a standard consequence of Helly’s theorem for convex sets in $\mathbb{R}^d$; see, for example, Chapter~3 of~\cite{gruber2007convex}.

\begin{lemma}
\label{lem:minimum-witness-cardinality}
Let $S \subset \R^d$ be a compact set. There exists a subset $W \subseteq S$ with $|W| \le d+1$ such that $\rad(W) = \rad(S)$. 
\end{lemma}
\begin{proof}
By definition of the Chebyshev radius, for every \(r < \rad(S)\) we have
\[
  \bigcap_{x\in S} B(x,r) = \emptyset .
\]
By Helly's theorem, there exist points \(x_0,\ldots,x_d \in S\) such that
\[
  \bigcap_{i=0}^{d} B(x_i,r) = \emptyset .
\]
Equivalently, \(\rad(\{x_0,\ldots,x_d\}) > r\).
Now set \(r_n := \rad(S) - \tfrac1n\). Applying the above with \(r=r_n\), we can
choose a \((d{+}1)\)-tuple
\[
  W_n := (x^{(n)}_0,\ldots,x^{(n)}_d)\in S^{d+1}
\]
such that \(\rad(W_n) > r_n\), and hence \(\rad(W_n) \ge \rad(S)-\tfrac1n\).
Since \(S\) is compact, \(S^{d+1}\) is compact as well, so \((W_n)\) has a convergent subsequence \(W_{n_k}\to W\in S^{d+1}\).
By continuity of the Chebyshev radius for finite point sets,
\[
\rad(W)\;\ge\;\limsup_{k\to\infty}\rad(W_{n_k})\;\ge\;\rad(S).
\]
On the other hand, \(W\subseteq S\) implies \(\rad(W)\le \rad(S)\). 

Therefore \(\rad(W)=\rad(S)\), as claimed.
\end{proof}

In particular, if $\text{rad}(S) \ge \Jung_d$, then $W$ is a witness. If $|W| < d+1$, a witness of cardinality exactly $d+1$ can be constructed by duplicating elements of $W$.

\subsection{Proof of Robust Jung Theorem}

\bigskip
\begin{rtheorem}{\ref{thm:robust-jung}}
Let $d \ge 1$ and let $S \subset \R^d$ be a compact set satisfying
\[
\diam(S) \le 1+\beta
\qquad\text{and}\qquad
\rad(S) \ge \Jung_d ,
\]
for some $0 \le \beta \le \beta_d$, where $\beta_d>0$ depends only on $d$.
Then, there exists a set of points
\[
\Delta = \{x_0,x_1,\dots,x_d\} \subset \R^d
\]
forming the vertex set of a regular simplex of edge length $1$ such that every witness \(W \subseteq S\) is \(C_d\beta\)-close to \(\Delta\),
for a constant $C_d>0$ depending only on $d$, i.e. \(\dist_H(\Delta, W) \le C_d \beta\).  
\end{rtheorem}

\bigskip

To prove this theorem, we first show that every witness is close (in Hausdorff distance) to the vertex set of a regular simplex. We then combine this with a stability statement for regular simplices (proved in Appendix~\ref{app:step_2}), which implies that any two such vertex sets must be
close to each other. Putting these two ingredients together yields the claim.

\begin{lemma}[The witnesses is nearly regular]
\label{lem:approximate_jung}
Suppose closed set $S \subseteq \R^d$ has diameter $1$ and Chebyshev radius $r \ge (1-\epsilon)\sqrt{\frac{d}{2(d+1)}}$. Then there exist $x_0,\dots,x_d \in S$ such that $\|x_i - x_j\|_2\ge 1 - O(d^2\epsilon)$ for all $i \neq j$.
\end{lemma}
\begin{proof}
Lemma~\ref{lem:minimum-witness-cardinality} implies that there exist $x_0,\dots,x_d \in S$ whose Chebyshev radius is $r$.

Suppose without loss of generality that the minimal enclosing ball is centered at the origin, and let $i_0,\dots,i_{n-1}$ be the indices of the $x_i$ such that $\|x_i\|_2= r$. The origin is contained in the convex hull of the $x_{i_s}$'s, say $\sum_s \lambda_s x_{i_s} = 0$, where $\lambda_s \ge 0$ and $\sum_s \lambda_s = 1$.

Write $\|x_{i_s} - x_{i_t}\|^2 = 1 - \delta_{st}$. For all $s$,
\[
 1 - \lambda_s = \sum_{t \ne s} \lambda_t = \sum_{t \ne s} \lambda_t (\|x_{i_s} - x_{i_t}\|^2 + \delta_{st}).
\]
We have
\[
 \sum_{t \neq s} \lambda_t \|x_{i_s} - x_{i_t}\|^2 = \sum_t \lambda_t \|x_{i_s} - x_{i_t}\|^2 = \sum_t \lambda_t (2r^2 - 2\langle x_{i_s}, x_{i_t} \rangle) = 2r^2,
\]
since $\sum_t \lambda_t x_{i_t} = 0$. Therefore
\begin{equation} \label{eq:lambda}
 1 - \lambda_s = 2r^2 + \sum_{t \neq s} \lambda_t \delta_{st}.
\end{equation}

Summing over all $s$,
\[
 n - 1 = 2nr^2 + \sum_{s \neq t} \lambda_t \delta_{st},
\]
and so
\[
 r^2 = \frac{n-1}{2n} - \frac{1}{2n} \sum_s \sum_{t \neq s} \lambda_t \delta_{st}.
\]
On the other hand, $r^2 \geq (1-2\epsilon)\frac{d}{2(d+1)}$. Since $r^2 \le \frac{n-1}{2n}$, if $(1-2\epsilon) \frac{d}{2(d+1)} > \frac{d-1}{2d}$ (which happens when $\epsilon < \frac{1}{2d^2}$) then $n = d+1$, and we can take $i_t = t$. Moreover,
\begin{equation} \label{eq:delta}
\sum_s \sum_{t \neq s} \lambda_t \delta_{st} = d - 2(d+1)r^2 \leq 2d\epsilon.
\end{equation}

Substituting \eqref{eq:delta} in \eqref{eq:lambda} yields
\[
 1 - \lambda_s \leq 2r^2 + 2d\epsilon \leq \frac{d}{d+1} + 2d\epsilon,
\]
hence
\[
 \lambda_s \geq \frac{1}{d+1} - 2d\epsilon.
\]
If $\epsilon \leq \frac{1}{2d(d+1)}$ then $\lambda_s \geq \frac{1}{2(d+1)}$. Since
\[
 \sum_s \sum_{t \neq s} \lambda_t \delta_{st} \geq \lambda_s \max_{t \neq s} \delta_{st},
\]
in view of \eqref{eq:delta} this means that $\delta_{st}  \leq 2d(d+1)\epsilon$ for all $s \neq t$.

Finally, $\|x_s - x_t\|_2\geq \sqrt{1 - 2d(d+1)\epsilon} \geq 1 - 2d(d+1)\epsilon$. This bound holds under the assumption $\epsilon \leq \frac{1}{2d(d+1)}$. Otherwise, the bound in the theorem trivially holds.
\end{proof}

\begin{lemma}[From nearly regular to regular]
\label{lem:step-1}
There exist constants \(c_d>0\) and \(\epsilon_d>0\), depending only on \(d\),
such that the following holds.
Let \(S=\{x_0,\ldots,x_d\}\subset \R^d\) satisfy
\[
  \diam(S)\le 1
  \qquad\text{and}\qquad
  \rad(S)\ge (1-\epsilon)\sqrt{\frac{d}{2(d+1)}}
\]
for some \(0<\epsilon\le \epsilon_d\).
Then there exists a regular simplex \(\Delta\) of diameter \(1\) such that
\[
  \dist_H(S,\Delta)\le c_d\,\epsilon .
\]
\end{lemma}
\begin{proof}
    We proceed by induction on the dimension \(d\). The base case \(d = 1\) is trivial.
    By the induction hypothesis and \Cref{lem:approximate_jung}, it suffices to prove the following: given points \(y_1, \ldots, y_d\) forming a regular simplex in a subspace of \(\R^d\) (with edge length \(1\)), assume there is a point \(y_0\) such that 
    \[
    \|y_0 - y_k\|^2 = 1 + f_k(\epsilon), \qquad f_k(\epsilon) = O(\epsilon).
    \]
    Then there exists a point \(y_\ast\) such that \(y_\ast, y_1, \ldots, y_d\) form a regular simplex and 
    \[
    \|y_\ast - y_0\|_2\le O\bigl(\epsilon).
    \]
    
    Let \(v\) be an affine functional with \(v\!\bigl(\operatorname{aff}\{y_1,\ldots,y_d\}\bigr)=0\) and \(v(y_0) > 0\), and let \(y_\ast\) be a point in the half-space \(\{x\in\R^d:\, v(x)>0\}\) that, together with \(y_1,\ldots,y_d\), forms a regular simplex.
    Define
    \[
        F\colon \R^d\to\R^d,\qquad
        F(x)=\bigl(\|x-y_1\|^2-1,\ldots,\|x-y_d\|^2-1\bigr).
    \]
    The \(i\)th row of \(DF(x)\) is \(2(x-y_i)^\top\). A direct computation at \(y_\ast\) yields
    \[
        DF(y_\ast)\,DF(y_\ast)^\top
        \;=\;4\bigl(\langle y_\ast-y_i,\;y_\ast-y_j\rangle\bigr)_{i,j}
        \;=\;2\,(I+J),
    \]
    where \(J\) is the all-ones matrix. Thus the singular values of \(DF(y_\ast)\) are \(\sqrt{2(d+1)}\) (once) and \(\sqrt{2}\) (multiplicity \(d-1\)), so \(DF(y_\ast)\) is invertible and
    \[
        \bigl\|DF(0)^{-1}\bigr\|
        \;=\;\frac{1}{\sigma_{\min}(DF(y_\ast))}
        \;=\;\frac{1}{\sqrt{2}}.
    \]
    
    Applying the inverse function theorem at \(y_\ast\) gives neighborhoods \(U \ni y_\ast\) and \(V \ni 0\) together with a local \(C^1\)-inverse \(G\).
    Hence, for all \(\epsilon\) sufficiently small, we have \(a_0 := F(y_0) \in V\).
    Note that there are at most two preimages of \(a_0\), since \(d\) spheres centered at \(y_1, y_2, \ldots, y_d\) can intersect in at most two points by an induction-on-dimension argument.
    As one of these preimages lies in the half-space where \(v<0\), while \(v(y_0)>0\), it follows that \(y_0 \in U\).

    By the multivariate version of Taylor’s theorem applied to \(G\), for every \(a \in V\) we can write
    \[
    G(a) = DF^{-1}(0)\, a + H(a)\, a, \qquad H\colon V \to \mathcal{M}(d \times d),
    \]
    where \(H\) is a continuous matrix-valued function with \(\lim_{a \to 0} H(a)=0\).
    Therefore,
    \[
    \|y_0 - y_\ast\|_2= \|G(a_0)-G(0)\|_2\;\le\; \bigl(\|DF(0)^{-1}\|_2+ \|H(a_0)\|\bigr)\cdot \sqrt{\sum_{k=1}^d f_k^2(\epsilon)} = O(\epsilon). 
    \]
\end{proof}

Finally we prove Robust Jung Theorem. Recall it's formulation

\bigskip
\begin{rtheorem}{\ref{thm:robust-jung}}
Let $d \ge 1$ and let $S \subset \R^d$ be a compact set satisfying
\[
\diam(S) \le 1+\beta
\qquad\text{and}\qquad
\rad(S) \ge \Jung_d ,
\]
for some $0 \le \beta \le \beta_d$, where $\beta_d>0$ depends only on $d$.
Then, there exists a set of points
\[
\Delta = \{x_0,x_1,\dots,x_d\} \subset \R^d
\]
forming the vertex set of a regular simplex of edge length $1$ such that every witness \(W \subseteq S\) is \(C_d\beta\)-close to \(\Delta\),
for a constant $C_d>0$ depending only on $d$, i.e. \(\dist_H(\Delta, W) \le C_d \beta\). 
\end{rtheorem}
\smallskip
\begin{proof}
Fix two witnesses \(W,W' \subseteq S\).
By Lemma~\ref{lem:step-1}, there exist regular simplices
\(\Delta,\Delta'\) of edge length \(1\) such that
\[
  \dist_H(W,\Delta) \le c_d\,\beta
  \qquad\text{and}\qquad
  \dist_H(W',\Delta') \le c_d\,\beta .
\]
Since \(W,W' \subseteq S\) and \(\diam(S)\le 1+\beta\), we have
\[
  \diam(\Delta \cup \Delta')
  \;\le\;
  \diam(W \cup W') + 2c_d\beta
  \;\le\;
  1 + \beta + 2c_d\beta .
\]
Shrinking \(\beta_d>0\) if necessary, we may assume
\(1 + \beta + 2c_d\beta \le 1 + \beta_d\).
Therefore, by Theorem~\ref{thm:two_regular_simplices},
\[
  \dist_H(\Delta,\Delta') \le (d+1)d^2\,\beta .
\]
Applying the triangle inequality for the Hausdorff distance,
\[
  \dist_H(W,W')
  \;\le\;
  \dist_H(W,\Delta)
  + \dist_H(\Delta,\Delta')
  + \dist_H(\Delta',W')
  \;\le\;
  \bigl(2c_d + (d+1)d^2\bigr)\beta .
\]
In particular, all witnesses are contained in a common
\(C_d\beta\)-neighborhood of \(\Delta\), for a constant
\(C_d>0\) depending only on \(d\).
This completes the proof.
\end{proof}

\section{Regular simplex rotation}
\label{app:step_2}

The aim of this section is to prove the following theorem.

\begin{theorem}
\label{thm:two_regular_simplices}
Let \(\Delta,\Delta' \subset \R^d\) be the vertex sets of regular simplices of
edge length \(1\). Suppose that
\[
  \diam(\Delta \cup \Delta') \le 1+\beta
\]
for some \(0 \le \beta < \beta_d\), where \(\beta_d>0\) depends only on \(d\).
Then
\[
  \dist_H(\Delta,\Delta') \le (d+1)d^2\,\beta .
\]
\end{theorem}

The proof uses Lie–algebra techniques together with several linear–algebraic properties of regular simplices.
First, in \Cref{lem:unique_regular} we show that \(\diam(\Delta \cup \Delta') = 1\) only if \(\Delta'=\Delta\).
Next, in \Cref{lem:diam_to_dist_regular_weak} we establish a weaker continuity statement: if \(\diam(\Delta \cup \Delta')\) is small, then \(\dist_H(\Delta, \Delta')\) is also small.
After that, in \Cref{lem:dist_to_motion_regular} we relate \(\dist_H(\Delta, \Delta')\) to the size of a Euclidean transformation \((\tau, R)\) (translation + rotation) carrying \(\Delta\) to \(\Delta'\); specifically, we show that for \(\dist_H(\Delta, \Delta') < \tfrac12\) there exists \((\tau, R)\) with
\[
\dist_H(\Delta, \Delta') \;>\; B_d\bigl(\|R-I_d\|_2+\|\tau\|\bigr),
\]
where \(B_d\) is a universal constant.
Finally, in \Cref{lem:dist_diam_linear_in_speed} assuming \(\|R-I_d\|_2+\|\tau\|\) is sufficiently small, we pass to the Lie–algebra level and express \(\dist_H(\Delta, \Delta')\) and \(\diam(\Delta \cup \Delta')\) in terms of the principal logarithm \((\tau_0, B)\in\mathfrak e(d)\); taking \(\diam(\Delta \cup \Delta')-1\) small enough then yields the claim.

\begin{remark}
\label{rem:regular_tight_frame}
Let \(\Delta_0\) be a regular unit simplex with vertices \(x_0, x_1, \ldots, x_d\).  
Placing the centroid of \(\Delta_0\) at the origin, the vertices form a \emph{tight frame} with frame constant \(\tfrac{1}{2}\).  
In other words, if \(X \in \R^{d \times (d+1)}\) denotes the matrix whose columns are the vectors \(x_i\), then
\[
XX^\top = \tfrac{1}{2}\, I_d.
\]
In particular, for any point \(y \in \R^d\) one has
\[
\|y\|^2 = \frac12 \sum \langle x_i, y \rangle^2
\]
\end{remark}

\begin{lemma}
\label{lem:unique_regular}
Let \(\Delta = \{x_0, x_1, \ldots, x_d\}\) and 
\(\Delta' = \{y_0, y_1, \ldots, y_d\}\) be two regular simplices of edge length \(1\). 
If 
\[
\diam(\Delta \cup \Delta') = 1,
\]
then \(\Delta = \Delta'\).
\end{lemma}
\begin{proof}
For convenience, scale by \(\sqrt2\). Thus we may assume both simplices have
edge length \(\sqrt2\), and
\[
\diam(\Delta\cup\Delta')=\sqrt2 .
\]
Let \(R:=\sqrt{\tfrac{d}{d+1}}\). A regular simplex of edge length \(\sqrt2\)
has circumradius \(R\), and its minimum enclosing ball (its circumball) is
unique.

By Jung's theorem,
\(\rad(\Delta\cup\Delta')\le R\), hence \(\Delta\cup\Delta'\) is contained in
some ball \(B(c,R)\). This ball contains \(\Delta\), so by uniqueness of the
circumball of \(\Delta\), its center \(c\) must be the circumcenter of \(\Delta\).
The same argument applied to \(\Delta'\) shows that \(c\) is also the circumcenter
of \(\Delta'\). Denote this common circumcenter by \(O\), and translate so that
\(O=0\).

Then \(\|x_i\|=\|y_j\|=R\) for all \(i,j\), and
\(\sum_{i=0}^d x_i=0\). In particular,
\[
\langle x_i,x_j\rangle =
\begin{cases}
R^2= \frac{d}{d+1}, & i=j,\\[2mm]
-\frac{1}{d+1}, & i\neq j,
\end{cases}
\]
and the simplex frame identity holds (see Remark~\ref{rem:regular_tight_frame}):
\[
\|y\|^2 \;=\; \sum_{i=0}^d \langle y, x_i\rangle^2
\qquad \text{for all } y\in\R^d.
\]
Fix \(j\) and set \(x:=y_j\). Since \(\diam(\Delta\cup\Delta')=\sqrt2\), we have
\(\|x-x_i\|\le \sqrt2\) for every \(i\). Expanding,
\[
\|x-x_i\|^2
= \|x\|^2+\|x_i\|^2-2\langle x,x_i\rangle
= 2R^2 -2\langle x,x_i\rangle
\le 2,
\]
so \(\langle x,x_i\rangle \ge R^2-1 = -\tfrac{1}{d+1}\) for all \(i\).

Define \(b_i := \langle x,x_i\rangle + \tfrac{1}{d+1}\ge 0\).
Using \(\sum_i x_i=0\) and \(\|x\|^2=R^2\),
\[
\sum_{i=0}^d b_i
= \sum_{i=0}^d \langle x,x_i\rangle + 1
= 1,
\]
and using the inner products of the simplex,
\[
\sum_{i=0}^d b_i^2
= \sum_{i=0}^d \langle x,x_i\rangle^2
+ \tfrac{2}{d+1}\sum_{i=0}^d \langle x,x_i\rangle
+ \tfrac{1}{d+1}
= \|x\|^2 + \tfrac{1}{d+1}
= 1.
\]
(Here we used the frame identity \ref{rem:regular_tight_frame} and \(\sum_i \langle x,x_i\rangle=0\).)

Since \(b_i\ge 0\) and \(\sum_i b_i=1\), we have \(\sum_i b_i^2\le 1\), with
equality iff exactly one \(b_i=1\) and the others are \(0\). Thus \(b_{i^\ast}=1\)
for some \(i^\ast\), i.e.
\(\langle x,x_{i^\ast}\rangle = R^2\).
Because \(\|x\|=\|x_{i^\ast}\|=R\), equality in Cauchy--Schwarz gives
\(x=x_{i^\ast}\). Hence every \(y_j\) equals some \(x_i\), so
\(\Delta' \subseteq \Delta\). By symmetry \(\Delta \subseteq \Delta'\), and
therefore \(\Delta=\Delta'\).
\end{proof}

\begin{definition}
\label{not:space_of_simplices}
Consider the subspace of regular unit simplices with ordered vertices:
\[
\cY=\Bigl\{(x_0,x_1,\ldots,x_d)\in(\R^d)^{d+1}\ \Big|\ \|x_i-x_j\|=1\ \text{for all } i\neq j\Bigr\}.
\]
This is a smooth algebraic submanifold of \((\R^d)^{d+1}\). The symmetric group \(S_{d+1}\) acts freely on \(\cY\) by permuting vertices; hence the quotient \(\cX:=\cY/S_{d+1}\) is a smooth manifold. Notice that on \(\cX\) the Hausdorff distance is well-defined, and the resulting metric topology agrees with the manifold topology.
\end{definition}

\begin{lemma}
\label{lem:diam_to_dist_regular_weak}
Assume \(\Delta_0\) and \(\Delta'\) are regular simplices with edge length \(1\).
For every \(\varepsilon > 0\) there exists \(\beta_\varepsilon > 0\) such that 
\[
\diam(\Delta_0 \cup \Delta') < 1+ \beta_\varepsilon
\]
implies 
\[
\dist_H(\Delta_0, \Delta') \le \varepsilon.
\]
\end{lemma}
\begin{proof}

Recall that \(\cX\) is the space of regular simplices with edge length \(1\), represented by their vertices (see \Cref{not:space_of_simplices}).

Define the function \(D_{\Delta_0}\colon \cX\to\mathbb{R}\) by 
\[ 
D_{\Delta_0}(\Delta') \;=\; \bigl(\diam(\Delta_0\cup \Delta')\bigr) - 1. 
\]
This function is continuous since 
\[ 
\diam(\Delta_0\cup \Delta') = \max\!\Bigl(1, \max_{i,j}\|x_i-x_j'\|\Bigr), 
\]
which is the maximum of finitely many continuous functions. Moreover, the equation 
\[ 
D_{\Delta_0}(\Delta)=0 
\] 
has the unique solution \(\Delta=\Delta_0\) by \Cref{lem:unique_regular}. 
Now fix any \(\beta_0>0\). The set of simplices 
\[ 
\{\Delta \in \cX : D_{\Delta_0}(\Delta)\le \beta_0\} 
\]
is compact. Hence, we can argue by compactness: for every \(\varepsilon>0\) there exists \(0<\beta\le \beta_0\) such that whenever \(D_{\Delta_0}(\Delta)\le \beta\) one has 
\[ 
\dist_H(\Delta,\Delta_0)\le \varepsilon. 
\] 
Indeed, suppose the contrary. Then there exists \(\varepsilon_0>0\) and a sequence of simplices \(\Delta_n\) such that 
\[ \diam(\Delta_n\cup \Delta_0)\le 1+\tfrac1n \qquad\text{and}\qquad \dist_H(\Delta_n,\Delta_0)\ge \varepsilon_0. 
\] 
By compactness, a subsequence \(\Delta_{n_k}\) converges to some \(\Delta^\star\). Passing to the limit we obtain 
\[ 
\diam(\Delta^\star\cup \Delta_0)=1, \qquad \dist_H(\Delta^\star,\Delta_0)\ge \varepsilon_0, 
\]
 which contradicts the uniqueness of the solution to \(D_{\Delta_0}(\Delta)=0\).
\end{proof}

\begin{lemma}
\label{lem:dist_to_motion_regular}
Let \(\Delta_0\) and \(\Delta'\) be two regular unit simplices.  
If \(\dist_H(\Delta_0, \Delta') < \tfrac{1}{2}\), then there exists a Euclidean transformation 
\((\tau, R) \in E(d)\), with translation part \(\tau \in \R^d\) and rotation part \(R \in O(d)\), such that
\[
\dist_H(\Delta_0, \Delta') \;\ge\; \dfrac{1}{2 \sqrt{d+1}} \,\bigl(\|R-I_d\|_2+\|\tau\|\bigr).
\]
\end{lemma}
\begin{proof}
Consider the Euclidean group of isometries of \(\R^d\), denoted \(E(d)\). It admits the semidirect product decomposition
\[
E(d)=\R^d \rtimes O(d), 
\qquad 
(\tau_1, A_1)\cdot(\tau_0, A_0)=\bigl(\tau_1 + A_1 \tau_0,\; A_1 A_0\bigr).
\]
Recall that \(\cY\) denotes the set of ordered regular \(n\)-simplices of unit edge length (see \Cref{not:space_of_simplices}):
\[
\cY=\Bigl\{(x_0,x_1,\ldots,x_d)\in(\R^d)^{d+1}\ \Big|\ \|x_i-x_j\|=1\ \text{for all } i\neq j\Bigr\}.
\]
Notice that \(E(d)\) acts on \(\cY\) by
\[
(\tau,A)\cdot(x_0,\ldots,x_d) = (Ax_0+\tau,\ldots,Ax_d+\tau).
\]
This is a free and transitive action; fixing a regular simplex \(\Delta_0\) (with ordered vertices) identifies \(E(d)\) diffeomorphically with \(\cY\).
Denote by \(\varepsilon\) the Hausdorff distance between the simplices \(\Delta_0\) and \(\Delta'\). 
Since the factorization map
\[
\cY \;\longrightarrow\; \cY/S_{d+1} \;=\; \cX
\]
is a covering and \(\varepsilon < \tfrac12\), there exists a permutation \(\sigma \in S_{d+1}\) such that the vertices of \(\Delta=\{x_0,\dots,x_d\}\) and \(\Delta'=\{x'_0,\dots,x'_d\}\) can be paired by
\[
x_i \;\mapsto\; x'_{\sigma(i)},
\qquad
\|x_i - x'_{\sigma(i)}\|_2\;\le\; \varepsilon
\quad\text{for all } i.
\]
Consider a Euclidean transformation \((\tau,R)\in E(d)\) such that \(R(x_i)+\tau = x_{\sigma(i)}'\). We will prove that 
\[
(d+1) \varepsilon^2 \ge \tfrac1{2}\|R-I_d\|^2_2 +\|\tau\|^2.
\]
By Remark~\ref{rem:regular_tight_frame}, the matrix \(X \in \R^{d \times (d+1)}\) formed by the columns \(x_i\) satisfies
\[
XX^\top = \tfrac{1}{2}\, I_d.
\]
Note also that for any matrix \(M\), one has the following connections between Frobenius and operator norms: 
\[
\|M\|_2 \le \|M\|_{\mathrm{Fr}} \le \sqrt d \|M\|_2. 
\]
By assumption, each column of
\[
Y := (R-I_d)X + \tau\cdot(1,\ldots,1)
\]
has norm at most \(\varepsilon\). Hence
\[
\|Y\|_\mathrm{Fr}^2 \le (d+1)\varepsilon^2.
\]
Since \(\sum_i x_i=0\) and \(\sum_i Rx_i=0\), the cross terms vanish, and therefore
\[
\|Y\|_\mathrm{Fr}^2
= \|(R-I_d)X\|_\mathrm{Fr}^2 + \|\tau\|^2
\;\le\; (d+1)\varepsilon^2.
\]
Using the tight-frame identity \(XX^\top=\tfrac12 I_d\), we obtain
\[
\|(R - I_d)X\|^2_2 \ge \|(R-I_d)X \cdot ((R-I_d)X)^\top\|_2
= \tfrac12\, \|(R-I_d)(R-I_d)^\top\|_2.
\]
Let \(M:=R-I_d\), and let \(\sigma_1,\ldots,\sigma_d\) be its singular values. Then \(MM^\top\) has \(\sigma_1^2, \ldots, \sigma_d^2\) singular values; since the operator norm of a matrix is the maximal singular value, one has 
\[
\|M\|_2^2=\|MM^\top\|_2.
\]
Combining the estimates, we obtain
\[
\begin{aligned}
(d+1)\,\varepsilon^2 
&\;\ge\; \|(R-I_d)X\|_\mathrm{Fr}^2 + \|\tau\|^2 \\[2pt]
&\;\ge\; \tfrac{1}{2}\,\|(R-I_d)(R-I_d)^\top\|_2 + \|\tau\|^2 \\[2pt]
&\;\ge\; \tfrac{1}{2}\,\|R-I_d\|_2^2 + \|\tau\|^2.
\end{aligned}
\]
Set \(a := \|R-I_d\|_2\) and \(b := \|\tau\|\). Then \((d+1)\varepsilon^2 \ge \tfrac12 a^2 + b^2\), and using \(\sqrt{x^2+y^2} \ge (x+y)/\sqrt 2\) (valid for \(x,y \ge 0\)),
\[
\sqrt{\tfrac12 a^2 + b^2}
\;\ge\; \frac{1}{\sqrt 2}\left(\frac{a}{\sqrt 2} + b\right)
\;\ge\; \frac{a+b}{2}.
\]
Hence
\[
\varepsilon \;\ge\; \frac{1}{\sqrt{d+1}}\sqrt{\tfrac12 a^2 + b^2}
\;\ge\; \frac{\|R-I_d\|_2 + \|\tau\|}{2\sqrt{d+1}}.
\]

\end{proof}
We work in the homogeneous (matrix) model of the Euclidean group:
\[
(\tau,R)\in E(d)\quad\longleftrightarrow\quad
\begin{pmatrix} R & \tau\\[0.2em] 0 & 1 \end{pmatrix},
\qquad
(\tau_1,R_1)\circ(\tau_2,R_2)=(\tau_1+R_1\tau_2,\,R_1R_2).
\]
The Lie algebra is
\[
(u,\Omega)\in\mathfrak e(d)\quad\longleftrightarrow\quad
\begin{pmatrix} \Omega & u\\[0.2em] 0 & 0 \end{pmatrix},
\qquad \Omega\in\mathfrak{so}(d),\ u\in\R^d.
\]
The matrix exponential satisfies
\[
\exp\!\begin{pmatrix} \Omega & u\\ 0 & 0 \end{pmatrix}
=
\begin{pmatrix}
e^{\Omega} & J(\Omega)\,u\\ 0 & 1
\end{pmatrix},
\qquad
J(\Omega):=\sum_{k\ge0}\frac{\Omega^k}{(k+1)!}.
\]
Equivalently,
\[
J(\Omega)=\int_{0}^{1} e^{t\Omega}\,dt
=\Omega^{-1}\bigl(e^{\Omega}-I\bigr),
\quad\text{with the last identity understood by power series if \(\Omega\) is singular.}
\]
Define \( j(z) := (e^{z}-1)/z \) with the convention \( j(0) := 1 \). Then, by the holomorphic functional calculus,
\[
\sigma\!\bigl(J(\Omega)\bigr) = j\!\bigl(\sigma(\Omega)\bigr),
\]
where \(\sigma\) denotes the spectrum of the matrix, that is, the multiset of its eigenvalues.

Hence
\[
J(\Omega)\ \text{invertible}\ \Longleftrightarrow\
\sigma(\Omega)\cap\bigl(2\pi i\,\ZZ\setminus\{0\}\bigr)=\varnothing.
\]
In particular, if \(\Omega\in\mathfrak{so}(d)\) has eigen-angles \(\{\theta_j\}\) (so \(\sigma(\Omega)=\{\pm i\theta_j\}\)),
then \(J(\Omega)\) is invertible iff \(\theta_j\notin 2\pi\ZZ\) for all \(j\). A convenient sufficient condition is
\(\|\Omega\|_2<\pi\) (equivalently, \(|\theta_j|<\pi\) for all \(j\)).

If \(R\) has no eigenvalue \(-1\), the principal matrix logarithm \(B=\log(R)\in\mathfrak{so}(d)\) is well-defined
and real-analytic, and we set
\[
\log(\tau,R):=(\tau_0',B),\qquad \tau_0':=J(B)^{-1}\tau.
\]
\begin{lemma}[Technical lemma for principal logarithm]
\label{lem:lie_group_vs_lie_algebra}
Define a neighborhood of \((0,I_d)\) in \(E(d)\) by
\[
U \;=\; \Bigl\{\,(\tau,R)\in E(d)\;\Bigm|\; \|R-I_d\|_2 < \sqrt{2} \,\Bigr\}.
\]
Then the principal logarithm
\[
\log : U \longrightarrow \mathfrak e(d), \qquad (\tau,R)\longmapsto (\tau_0',B)
\]
is well-defined with \(\|B\|_2\le \tfrac{\pi}{2}\).

Moreover,
\begin{equation}\label{eq:sum-bounds}
\|R-I_d\|_2 + \|\tau\|
\;\le\; \|B\|_2 + \|\tau_0'\|
\;\le\; \tfrac{\pi}{2}\,\bigl(\|R-I_d\|_2 + \|\tau\|\bigr).
\end{equation}
\end{lemma}

\begin{proof}
We again use the technique of eigen-angles of orthogonal matrices. Over \(\mathbb{C}\), the eigenvalues of \(R\) are \(e^{\pm i\theta_j}\) with \(0\le \theta_j<\pi\). Equivalently, in an orthonormal basis \(R\) is block-diagonal with \(2\times2\) rotation blocks
\[
\begin{pmatrix}\cos\theta&-\sin\theta\\[2pt]\sin\theta&\cos\theta\end{pmatrix}
\]
(and acts as the identity on the orthogonal complement). 
We define \(B\) in the same orthonormal basis by  \(2\times2\) skew blocks
\[
\begin{pmatrix}0&-\theta\\[2pt]\theta&0\end{pmatrix}
\]
(and is zero on the orthogonal complement).
Since \(\log:U\to V:=\mathrm{Im}(\log)\) is bijective and \(D\exp_{(\tau_0', B)}\) is invertible for all \((\tau_0', B)\in V\), the inverse function theorem makes \(\exp:V\to U\) a bijective local diffeomorphism, hence its inverse \(\log\) is smooth and \(\log:U\to V\) is a diffeomorphism.

Since \(R-I_d\) is normal,
\[
\|R-I_d\|_2=\max_j|e^{i\theta_j}-1|=2\max_j\sin\!\bigl(\tfrac{\theta_j}{2}\bigr)
=2\sin\!\bigl(\tfrac{\|B\|_2}{2}\bigr),
\]
which yields 
\[
\|R-I_d\|_2 \;\le\; \|B\|_2 \;\le\; \frac{\pi}{2}\,\|R-I_d\|_2
\]
from \(\sin t\le t\) and \(\sin t\ge \tfrac{2}{\pi}t\) on \([0,\tfrac{\pi}{2}]\) (hence also on \([0,\tfrac{\pi}{4}]\)).

On each such \(2\)-dimensional invariant subspace,
\[
J(B):=(e^{B}-I)B^{-1}\quad\text{acts by}\quad \frac{e^{i\theta}-1}{i\theta}
\;=\;e^{i\theta/2}\,\frac{2\sin(\theta/2)}{\theta},
\]
so \(\|J(B)\|_2\le 1\).
Moreover,
\[
\|J(B)^{-1}\|_2\le \sup_{\theta\in[0,\|B\|_2]}\frac{\theta}{2\sin(\theta/2)}
\le \dfrac{\pi}{2\sqrt{2}}
\]
giving 
\[
\|\tau\|_2\;\le\; \|\tau_0'\|_2\;\le\; \frac{\pi}{2\sqrt{2}}\;\|\tau\|
\]
Consequently,
\begin{equation*}
\|R-I_d\|_2+\|\tau\|_2\;\le\; \|B\|_2+\|\tau_0'\|
\;\le\; \frac{\pi}{2}\,\bigl(\|R-I_d\|_2+\|\tau\|\bigr).
\end{equation*}
\end{proof}

\begin{lemma}
\label{lem:dist_diam_linear_in_speed}
There exists \(\varepsilon_0 > 0\) such that for any Euclidean transformation 
\[
R_\tau := (\tau, R) \in E(d), \qquad \|\tau\|_2+ \|R - I_d\|_2 \le \varepsilon_0,
\]
one has 
\[
\dist_H\!\bigl(\Delta_0, R_\tau[\Delta_0]\bigr) 
\;<\; d^2 (d+1)\,\bigl(\diam(\Delta_0 \cup R_\tau[\Delta_0]) - 1\bigr).
\]
\end{lemma}
\begin{proof}
Since the proof is technical, we split it into several steps. 
The guiding idea is this: any Euclidean motion near the identity can be written as
\(\exp\bigl(t(\tau_0,A)\bigr)\) with a \emph{normalized} Lie–algebra direction
\(\|A\|_2+\|\tau_0\|=1\) and a small scalar \(t>0\).
We then show that, regardless of the chosen normalized direction \((\tau_0,A)\),
both the Hausdorff distance \(\dist_H(\Delta_0,\Delta')\) and the diameter
\(\diam(\Delta_0\cup\Delta')-1\) are linear in \(t\) to first order, with uniform
\(O(t^2)\) remainders; moreover, the corresponding linear coefficients are
\emph{uniformly bounded away from zero}. This uniform control allows us to absorb
the quadratic errors and conclude the desired inequality.

\medskip
\textbf{Step 1: Reduction to the Lie--algebra level and normalization.}

By Lemma~\ref{lem:lie_group_vs_lie_algebra} we may work directly at the Lie--algebra level.  
Let \((\tau_0', B) \in \mathfrak e(d)\) with
\[
\|B\|_2 + \|\tau_0'\|_2\;\le\; \varepsilon_0,
\]
and set \(R_\tau := \exp\bigl((\tau_0', B)\bigr)\).  
It then suffices to prove that
\[
\dist_H\!\bigl(\Delta_0, R_\tau[\Delta_0]\bigr) 
\;<\; d^2(d+1)\,\bigl(\diam(\Delta_0 \cup R_\tau[\Delta_0]) - 1\bigr).
\]
Writing \(t:=\|\tau_0'\|+\|B\|_2>0\) and \((\tau_0,A):=(\tau_0',B)/t\), we obtain the normalized form
\[
(\tau,R)=\exp\!\bigl(t(\tau_0,A)\bigr),\qquad \|A\|+\|\tau_0\|=1.
\]
Denote \(R_\tau(\Delta_0)\) by \(\Delta(t)\).

\medskip \textbf{Step 2: First--order expansions for vertices and the translation.}

From now on, we assume \(t \le \tfrac12\). Expanding the integral, one obtains
\[
\tau =t\,\tau_0+\tfrac{t^2}{2}\,A\tau_0+ t^3 \cdot \tau_t,
\qquad \|\tau_t\|_2<\cfrac16.
\]
For each vertex \(x_i\) set
\[
y_i(t) \;:=\; e^{tA}x_i + \xi(t).
\]
Introduce the notation
\[
u_{ij} := x_i - x_j, 
\qquad 
v_i := Ax_i + \tau_0.
\]
Since we normalized \((\tau_0, A)\) so that \(\|A\|_2+ \|\tau_0\|_2= 1\), and because
\(\|x_i\|_2= \sqrt{\tfrac{d}{2(d+1)}}\),
we have the uniform bound
\[
\|v_i\|_2\le 1.
\]
Using the Taylor expansion
\[
e^{tA}x_j \;=\; x_j + t\,Ax_j + \tfrac{t^2}{2}\,A^2x_j + t^3 x_j(t),
\qquad \|x_j(t)\|\le \cfrac16,
\]
together with the expansion of \(\xi(t)\), we obtain
\[
y_j(t)-x_i
= -u_{ij} \;+\; t\,v_j \;+\; \tfrac{t^2}{2}\,A v_j \;+\; t^3 v_j(t), \qquad \|v_j(t)\|_2\le \cfrac13.
\]
\medskip \textbf{Step 3: Expansions for \(\dist_H\) and \(\diam\) with uniform remainders.}
Observe that \(\|y_i(t)-x_i\|^2 = \langle y_i(t)-x_i, y_i(t)-x_i\rangle\). A direct calculation yields for \(i=j\) (since \(u_{ii}=0\)),
\[
\|y_i(t)-x_i\|^2=\|v_i\|^2\,t^2+ r_{i}(t) \cdot t^4, \qquad |r_i(t)| \le 1 
\]
and for \(i\neq j\),
\[
\|y_j(t)-x_i\|^2
= 1 + 2\langle -u_{ij},\,v_j\rangle\,t + f_{ji}(t) \cdot t^2, \qquad |f_{ji}(t)| \le 3.
\]
Hence, for all sufficiently small \(t\) (e.g., \(|t|\le 1\) under the normalization \(\|A\|+\|\tau_0\|=1\)),
\[
\dist_H\bigl(\Delta_0,\Delta(t)\bigr)
= \Bigl(\max_i \|v_i\|\Bigr)\,t \;+\; r_{\mathrm{dist}}(t)\,t^2,
\qquad |r_{\mathrm{dist}}(t)| \le 2,
\]
and
\[
\diam\bigl(\Delta_0 \cup \Delta(t)\bigr)
= 1 \;+\; \Bigl(\max_{i,j} \langle -u_{ij},\,v_j\rangle\Bigr)\,t \;+\; r_{\mathrm{diam}}(t)\,t^2,
\qquad |r_{\mathrm{diam}}(t)| \le 2.
\]
\medskip \textbf{{Step 4: Coercivity linking \(\max_j\|v_j\|\) and \(\max_{i,j}\langle -u_{ij},v_j\rangle\).}}
Next, we show that 
\[
\max_j \|v_j\|_2\;<\; (d+1)d^2 \cdot \max_{i,j}\,\langle -u_{ij},\,v_j\rangle.
\]
Note that for any regular simplex 
\(\{x_1,\ldots,x_d\}=\Delta''\subset \R^{d-1}\) with edge length \(1\) and centroid at the origin, 
for any \(s\in\R^{d-1}\),
\[
\sum_{i=1}^d \langle s,x_i\rangle^2=\tfrac12\|s\|^2.
\]
This follows once more from Remark~\ref{rem:regular_tight_frame}: the vertices of a regular simplex form a tight frame with frame constant \(\tfrac{1}{2}\).

Let us fix the vertex \(x_j\) and denote \(\tfrac{x_j}{\|x_j\|}\) by \(e_j\). Now let us decompose any vector into parts lying at \(\R \cdot  e_j\) and \(e_j^\perp\):
\[
v_j = Ax_j +\tau_0, \qquad v_j = a \cdot e_j + v_j^\perp, \quad v_j^\perp \in e_j^\perp
\]
and 
\[
u_{ij} = \gamma e_j +  t_i, \qquad t_i \in e_j^\perp, \quad\gamma = - \sqrt{\tfrac{d+1}{2d}}
\]
Note that the vectors \(t_i\) form a regular simplex with edge \(1\) in the \(d-1\) dimensional subspace of \(\R^d\). Hence
\[
\|v_j^\perp \|^2 = 2 \cdot \sum_{i \not=j} \langle v_j^\perp, t_i\rangle^2 =  2 \cdot \sum_{i \not=j} \langle v_j, u_{ij}\rangle^2 - a^2 \cdot (d+1),
\]
and hence
\[ 
\|v_j\|^2 = 2 \cdot \sum_{i \not= j } \langle v_j, - u_{ij}\rangle^2 - d\cdot a^2.
\]
Observe that
\[
\sum_{j}\sum_{i\ne j}\langle v_j,\,-u_{ij}\rangle
=\sum_{j}\sum_{i\ne j}\langle v_j,\,x_j-x_i\rangle
=(d+1)\sum_{j}\langle v_j,\,x_j\rangle
=(d+1)\sum_{j}\langle \tau_0,\,x_j\rangle
=0,
\]
since \(\langle Ax_j, x_j\rangle=0\)\footnote{This is because \(A \in \mathfrak{so}(d)\) which means it is skew-symmetric.} and \(\sum_j x_j=0\).

Partition the multiset \(\{\langle v_j,-u_{ij}\rangle : j,\ i\ne j\}\) into positives
\(a_1,\ldots,a_k>0\) and negatives \(b_1,\ldots,b_m<0\), with \(k+m=(d+1)d\).
Because
\[
\sum_{\ell=1}^{k} a_\ell \;=\; \sum_{r=1}^{m} |b_r|,
\]
we have
\[
k\,\max_\ell a_\ell \;\ge\; \sum_{\ell=1}^{k} a_\ell
\;=\; \sum_{r=1}^{m} |b_r|
\;\ge\; \max_r |b_r|
\quad\Longrightarrow\quad
k^2\bigl(\max_\ell a_\ell\bigr)^2 \;\ge\; \bigl(\max_r |b_r|\bigr)^2.
\]
Hence, with \(N:=(d+1)d\),
\[
N^2 \,\bigl(\max_{i,j}\langle v_j,\,-u_{ij}\rangle\bigr)^2
\;\ge\;
\max_{i,j}\langle v_j,\,-u_{ij}\rangle^2.
\]
Consequently,
\[
\|v_j\|^2
\;\le\; 2d\,\max_{i,j}\langle v_j,\,-u_{ij}\rangle^2
\;\le\; 2d\,N^2\,\bigl(\max_{i,j}\langle v_i,\,-u_{ij}\rangle\bigr)^2,
\]
and therefore
\[
\max_j \|v_j\|
\;<\; (d+1)d^2 \;\max_{i,j}\langle v_i,\,-u_{ij}\rangle.
\]
\medskip \textbf{Step 5: Lower control on \(\max_j\|v_j\|\)}.

We estimate \(\max_j \|v_j\|\) using the Frobenius/operator–norm relation and
\(XX^\top=\tfrac12 I_d\).

As before,
\[
(d+1)\,\max_j \|v_j\|^2 \;\ge\; \sum_{j=0}^d \|v_j\|^2
\;=\; \|AX+\tau_0\,\mathbf 1^\top\|_{\mathrm{Fr}}^2
\;=\; \|AX\|_{\mathrm{Fr}}^2 \;+\; (d+1)\,\|\tau_0\|^2,
\]
where the cross term vanishes since \(\sum_j x_j=0\).

On the other hand,
\[
\|AX\|_{\mathrm{Fr}}^2
= \operatorname{tr}\bigl(X^\top A^\top A X\bigr)
= \operatorname{tr}\bigl(A\,XX^\top A^\top\bigr)
= \tfrac12\,\operatorname{tr}(A A^\top)
= \tfrac12\,\|A\|_{\mathrm{Fr}}^2
\;\ge\; \tfrac12\,\|A\|_2^2.
\]
Hence
\[
\bigl(\max_j \|v_j\|\bigr)^2
\;\ge\; \|\tau_0\|^2 \;+\; \frac{1}{2(d+1)}\,\|A\|_2^2.
\]
With the normalization \(\|A\|_2+\|\tau_0\|=1\), write \(a:=\|A\|_2\), \(b:=\|\tau_0\|\) with
\(a,b\ge 0\) and \(a+b=1\). Then
\[
\bigl(\max_j \|v_j\|\bigr)^2
\;\ge\; b^2 + \frac{1}{2(d+1)}\,a^2
\;\ge\; \min_{a+b=1}\Bigl(b^2 + \frac{1}{2(d+1)}\,a^2\Bigr)
= \frac{1}{2d+3},
\]
so
\[
\ \max_j \|v_j\|_2\;\ge\; \frac{1}{\sqrt{2d+3}}.
\]
\medskip \textbf{Uniform choice of \(t_0\).}
Set \(\alpha:=\max_{i,j}\langle -u_{ij},v_j\rangle\).
By Steps~4–5 we have the uniform lower bound
\[
\alpha \;\ge\; \frac{1}{(d+1)d^2\,\sqrt{2d+3}}.
\]
Fix
\[
t_0 :=  \frac{1}{4\,(d+1)d^2\,\sqrt{2d+3}}.
\]
Then for every \(t\in(0,t_0]\) and every normalized pair \((A,\tau_0)\),
the quadratic remainders are uniformly bounded by the linear terms:
\[
2t^2 \;\le\; \tfrac12\,\alpha t
\quad\text{and}\quad
2t^2 \;\le\; \tfrac12\,(d+1)d^2\,\alpha t.
\]
Hence
\[
\diam(\Delta_0\cup\Delta(t)) - 1
\;\ge\; \alpha t - 2t^2 \;\ge\; \tfrac12\,\alpha t,
\]
and
\[
\dist_H(\Delta_0,\Delta(t))
\;\le\; (d+1)d^2\,\alpha t + 2t^2
\;\le\; (d+1)d^2\,\alpha t + \tfrac12\,(d+1)d^2\,\alpha t.
\]
Combining two inequalities yeilds
\[
\dist_H\bigl(\Delta_0,\Delta(t)\bigr)
\;\le\; (d+1)d^2\,\bigl(\diam(\Delta_0\cup\Delta(t))-1\bigr),
\]
for all \(t\in(0,t_0]\). 
\end{proof}

\section{Covering number of the sphere}
\label{app:covering-number}

We work on the unit sphere $S^{d-1}\subset\R^d$ with the angular metric
\[
\rho(x,y)\ :=\ \angle(x,y)\ :=\ \arccos(\langle x,y\rangle)\in[0,\pi].
\]
For $\alpha>0$, let $\cN_\angle(S^{d-1},\alpha)$ be the smallest $N$ such that
$S^{d-1}$ can be covered by $N$ angular balls of radius $\alpha$ (equivalently,
by spherical caps of angular radius $\alpha$). We write $|S^{k}|$ for the surface
area of $S^{k}$.

\begin{lemma}[Bounds on the covering number]
For every $d\ge 2$ and every $\alpha\in(0,\pi/2]$,
\begin{equation}
\label{eq:covering-explicit}
  \sqrt d \,\alpha^{-(d-1)}
  \;<\;
  \cN_\angle(S^{d-1},\alpha)
  \;<\;
  2^{2d}\,\alpha^{-(d-1)}.
\end{equation}
\end{lemma}

\begin{proof}
Fix $\alpha\in(0,\pi/2]$. For a center $u\in S^{d-1}$ define the spherical cap
\[
C_\alpha(u)\ :=\ \{x\in S^{d-1}:\rho(x,u)\le \alpha\}. 
\]
By symmetry, the surface area of $C_\alpha(u)$ does not depend on $u$; we denote
it by $|C_\alpha|$.

Its surface area can be computed using spherical coordinates in a standard way
(see e.g.\ \cite{Li2011HypersphericalCap}):
\[ 
|C_\alpha|
= \frac{2\pi^{(d-1)/2}}{\Gamma(\tfrac{d-1}2)}
  \int_{0}^\alpha \sin^{d-2} \theta\, d \theta.
\]
The surface area of the sphere is
\[
|S^{d-1}| = \frac{2 \pi^{d/2}}{\Gamma(d/2)},
\]
where \(\Gamma\) denotes the Gamma function.

We claim that for every $\alpha\in(0,\pi/2]$,
\begin{equation}
\label{eq:area-sandwich}
\frac{|S^{d-1}|}{|C_\alpha|}
\ \le\
\cN_\angle(S^{d-1},\alpha)
\ \le\
\frac{|S^{d-1}|}{|C_{\alpha/2}|}.
\end{equation}

Indeed, if $S^{d-1}\subseteq\bigcup_{i=1}^N C_\alpha(u_i)$, then by subadditivity
of area,
\[
|S^{d-1}|
\le \sum_{i=1}^N |C_\alpha(u_i)|
= N|C_\alpha|,
\]
hence $N\ge |S^{d-1}|/|C_\alpha|$.

For the upper bound, take a maximal $\alpha$-separated set
$\{x_1,\dots,x_m\}\subset S^{d-1}$, so that $\rho(x_i,x_j)>\alpha$ for $i\neq j$.
Maximality implies that $\{C_\alpha(x_i)\}_{i=1}^m$ covers $S^{d-1}$, hence
$\cN_\angle(S^{d-1},\alpha)\le m$. Moreover, the caps $C_{\alpha/2}(x_i)$ are
pairwise disjoint, so
\[
m\,|C_{\alpha/2}|\le |S^{d-1}|,
\]
and therefore $m\le |S^{d-1}|/|C_{\alpha/2}|$. This proves
\Cref{eq:area-sandwich}.

We now derive explicit bounds. Writing
\[
\frac{|S^{d-1}|}{|C_\alpha|}
= \frac{\sqrt \pi\,
       \Gamma(\frac{d-1}{2})/\Gamma(\frac{d}{2})}
       {\int_{0}^\alpha \sin^{d-2} \theta\, d \theta},
\]
and using that for $0\le \theta\le \pi/2$,
\[
\frac{\theta}{2}\le \sin\theta \le \theta,
\]
we obtain
\[
\frac{\alpha^{d-1}}{(d-1)2^{d-2}}
\;<\;
\int_{0}^\alpha \sin^{d-2} \theta\, d \theta
\;<\;
\frac{\alpha^{d-1}}{d-1}.
\]
Hence
\[
\frac{|S^{d-1}|}{|C_\alpha|}
>
\frac{\sqrt \pi (d-1)\,
      \Gamma(\frac{d-1}{2})/\Gamma(\frac{d}{2})}
     {\alpha^{d-1}},
\]
and
\[
\frac{|S^{d-1}|}{|C_{\alpha/2}|}
<
\frac{\sqrt \pi (d-1) 2^{2d-3}\,
      \Gamma(\frac{d-1}{2})/\Gamma(\frac{d}{2})}
     {\alpha^{d-1}}.
\]
Let $x=\frac{d-1}{2}$. Gautschi’s inequality implies that for all $x>\tfrac12$,
\[
\frac{1}{\sqrt{x+\tfrac12}}
\ \le\
\frac{\Gamma(x)}{\Gamma(x+\tfrac12)}
\ \le\
\frac{1}{\sqrt{x-\tfrac12}}.
\]
Substituting $x=\frac{d-1}{2}$ yields
\[
\sqrt{\frac{2}{d}}
\ \le\
\frac{\Gamma(\frac{d-1}{2})}{\Gamma(\frac d2)}
\ \le\
\sqrt{\frac{2}{d-2}}.
\]
Combining these inequalities gives
\[
\sqrt d\,\alpha^{-(d-1)}
\;<\;
\cN_\angle(S^{d-1},\alpha)
\;<\;
2^{2d}\,\alpha^{-(d-1)},
\]
as claimed.
\end{proof}

\section{Simplex rotation lemmas}
\label{app:rotation_lemmas}

This appendix collects two technical lemmas used in the lower-bound proof of Main Theorem~\ref{thm:excess-bounds}. Together they realize the adversary's rotation step: given a direction \(v\) along which the current simplex \(\Delta\) is nearly aligned, Lemma~\ref{lem:simplex_rotation} constructs an explicit isometry \(R_\theta\) that tilts \(\Delta\) away from \(v\) by a controlled angle, and Lemma~\ref{lem:rotating_neighborhood} controls how far the resulting rotated simplex moves in the ambient space.

\begin{lemma}[Simplex rotation]
\label{lem:simplex_rotation}
    Consider a regular simplex \(\Delta = \{0, x_1, \ldots, x_d\}\) of edge length \(1\), an angle \(\theta \le \pi/18\), and a direction \(v\).
    Without loss of generality, assume the edge \(0x_1\) satisfies
    \[
        \langle v, x_1 - 0\rangle = \max_{i,j} \langle v, x_i - x_j\rangle,
        \qquad
        \langle v, x_1\rangle > \cos\theta .
    \]
    Then there exists an isometry \(R_\theta:\R^d\to\R^d\) such that the rotated simplex \(\Delta' := R_\theta \Delta\) satisfies
    \[
        \langle v, x_1' - 0\rangle = \max_{i,j} \langle v, x_i' - x_j'\rangle,
        \qquad
        x_i' := R_\theta x_i .
    \]
    Moreover,
    \[
        \langle v, x_1' - 0\rangle \le \cos\theta .
    \]
\end{lemma}

\begin{proof}
Consider the rotation \(R_\theta:\R^d\to\R^d\) that acts in the plane \(\Pi=\operatorname{span}\{v,x_1\}\) as a rotation by angle \(\theta\) from \(v\) toward \(x_1\) along the smaller angle between them, and acts as the identity on the orthogonal complement \(\Pi^\perp\).
If \(v=x_1\), take any plane \(\Pi\ni x_1\) and either of the two possible rotations.

In an orthonormal basis \(d_1,\ldots,d_d\) with \(\operatorname{span}\{d_1,d_2\}=\Pi\), \(d_1=v\), and
\[
    x_1 = k_1 d_1 + k_2 d_2 \quad \text{with } k_1,k_2>0,
\]
the map \(R_\theta\) is represented by
\begin{equation}\label{eq:rotation}
R_{\theta}
=
\begin{pmatrix}
\cos\theta & -\sin\theta & 0 & \cdots & 0 \\
\sin\theta &  \cos\theta & 0 & \cdots & 0 \\
0          &  0          & 1 &        & 0 \\
\vdots     &  \vdots     &   & \ddots &   \\
0          &  0          & 0 &        & 1
\end{pmatrix}.
\end{equation}
First, note that
\[
    1 \ge k_1=\langle x_1,v\rangle \ge \cos \theta, \qquad 0 \le k_2=\sqrt{1-\langle x_1,v\rangle^2} \le \sin \theta,
\]
since \(\langle d_1,x_1\rangle=k_1\) and \(k_1^2+k_2^2=1\).
Also, \(\langle x_1, x_i - x_j\rangle \in \{0,\pm\tfrac12\}\).

A direct calculation gives
\[
    \cos 2\theta \;=\; \cos^2\theta - \sin^2\theta \;\le\; \langle x_1', v\rangle \;\le\; \cos\theta .
\]
Set \(w' := v - x_1'\).
Then \(\|w'\|_2 \le \sqrt{2-2\cos \theta}\), and hence
\[
    \langle v, x_i' - x_j' \rangle
    \;=\;
    \langle x_1' + w',\, x_i' - x_j' \rangle
    \;\le\; \tfrac12 + \|w'\|_2
    \;\le\; \cos 2\theta
    \;\le\; \langle x_1', v\rangle
\]
for \(\theta \le \pi/18\). The case \(x_1 = v\) is identical.
This proves the claim.
\end{proof}

\begin{lemma}[Rotation and neighborhood of the simplex]
\label{lem:rotating_neighborhood}
    Fix two positive numbers \(\alpha, \alpha^\prime\).
    Assume the isometry \(R_\theta\) from \Cref{eq:rotation} satisfies
    \[
        B(R_\theta x_1, \alpha^\prime) \subset B(x_1, \alpha).
    \]
    Denote \(\Delta' := R_\theta(\Delta)\).
    Then
    \[
        \Delta' + B({\alpha'}) \subset \Delta+B(\alpha).
    \]
\end{lemma}
\begin{proof}
    To ensure the inclusion \( \Delta' + B({\alpha'}) \subset \Delta+B(\alpha) \), it suffices to check that no vertex of the simplex moves by more than \( \alpha - \alpha' \) under the rotation \(R_{\theta}\).

    Fix any vertex \(x_i \neq x_1\), and decompose it as \(x_i = w+p\), where \(w \in \Pi\) and \(p \in \Pi^\perp\).
    Note that \(\|w\|_2 \le 1\), since \(\|x_i\|_2^2 = \|w\|_2^2 + \|p\|_2^2 = 1\).
    Moreover,
    \[
        x_i' = R_\theta[w] + p,
    \]
    so
    \[
        \|x_i - x_i'\|_2^2 
        = \|w - R_\theta[w]\|_2^2
        = 2\|w\|_2^2 - 2 \langle w, R_\theta[w]\rangle
        = 2\|w\|_2^2\,(1 - \cos\theta).
    \]
    Since
    \[
        \|x_1 - x_1'\|_2^2 = 2(1-\cos\theta)
    \]
    and \(\|w\|_2\le 1\), we obtain
    \[
        \|x_i - x_i'\|_2\le \|x_1 - x_1'\|_2.
    \]
    By the assumption of the lemma, vertex \(x_1\) moves by at most \(\alpha - \alpha'\), 
    and hence every other vertex does as well. 
    This proves the claim.
\end{proof}

\end{document}